\documentclass[11pt,a4paper]{article}
\usepackage[hyperref]{emnlp2020}
\usepackage{times}
\usepackage[utf8]{inputenc} % allow utf-8 input
\usepackage[T1]{fontenc}    % use 8-bit T1 fonts
\usepackage{hyperref}       % hyperlinks
\usepackage{url}            % simple URL typesetting
\usepackage{booktabs}       % professional-quality tables
\usepackage{amsfonts}       % blackboard math symbols
\usepackage{nicefrac}       % compact symbols for 1/2, etc.
\usepackage{microtype}      % microtypography

\aclfinalcopy

\usepackage{amsmath}
\usepackage{times}
\usepackage{url}
\usepackage{morefloats}
\usepackage{graphicx}
\usepackage{pgfplots}
\pgfplotsset{compat=1.3}
\usetikzlibrary{shapes,shadows,positioning,calc,patterns}
\usepackage{multirow}
\usepackage{picinpar}
\usepackage{changepage}
\usepackage{epstopdf}
\usepackage{breqn}
\usepackage{verbatim}
\usepackage{wrapfig}
\usepackage{subcaption}

\usepackage[ruled,vlined]{algorithm2e}

\input{Definitions}

\def\UNA{$\mathrm{UNAS}$}
\def\DNA{$\mathrm{DNAS}$}
\def\HNA{$\mathrm{HNAS}$}

\newcommand{\squad}{SQuAD\xspace}

\title{Attention that does not Explain Away}
\author{Nan Ding \\ Google \\\And Xinjie Fan \\ UT Austin \\\And Zhenzhong Lan \\ Google \\\And Dale Schuurmans \\Google \\\And Radu Soricut \\Google}

\begin{document}
\maketitle

\begin{abstract}
Models based on the Transformer architecture have achieved better accuracy than the ones based on competing architectures for a large set of tasks.
A unique feature of the Transformer is its universal application of a self-attention mechanism, which allows for free information flow at arbitrary distances.
Following a probabilistic view of the attention via the Gaussian mixture model, we find empirical evidence that the Transformer attention tends to ``explain away'' certain input neurons.
To compensate for this, we propose a doubly-normalized attention scheme that is simple to implement and provides theoretical guarantees for avoiding the ``explaining away'' effect without introducing significant computational or memory cost.
Empirically, we show that the new attention schemes result in improved performance on several well-known benchmarks.
\end{abstract}

\section{Introduction}
The Transformer architecture~\citep{vaswani2017attention} has been successfully used to improve state-of-the-art performance in a variety of machine learning tasks, such as machine translation~\citep{vaswani2017attention,dehghani2019universal}, language modeling~\citep{devlin2018bert,yang2019xlnet}, summarization~\citep{cohan2018discourseaware,goodman2019multistage}, dialog~\citep{mazare2018training,cheng2019dynamic}, image captioning~\citep{sharma2018conceptual,zhao2019informative}, and visual question answering~\citep{yu2019deep,tan2019lxmert}.
One of the most important components of the Transformer architecture is its self-attention mechanism, applied universally to both the encoder and the decoder components.
This attention mechanism allows for information to freely flow between inputs at arbitrary distances, which is intuitively appealing for modeling natural language or tasks that need to model cross-modal relationships between their inputs.

Despite the empirical success of the self-attention mechanism, little formal work has been done to analyze its statistical properties and relate it to previously known classical models.
Better understanding its properties can lead to insights into what it does and does not do well.
This in turn can lead to improvements to the attention mechanism and ultimately to a better-performing Transformer network.

In this paper, we closely study the Transformer attention formulation from a probabilistic view via the Gaussian mixture model.
If we consider the Transformer model as a stack of layers with data flowing from lower to upper layers, then the output neurons (from the upper layer) of an attention unit can be regarded as the most likely data generated by a Gaussian mixture model (GMM), while the input neurons (from the lower layer) of the attention unit act as the Gaussian centers.

Our insight here is that this Transformer attention scheme has an ``explaining away'' effect, which means that the information present in certain lower layer neurons may be filtered out completely.
This is because for a GMM, not all Gaussian centers (lower layer neurons) are required to contribute in generating output data (upper layer neurons).
The information of the centers that do not generate data is lost after observing the data. This "explaining-away" effect is related to the one in the directed graphical model, in the sense that the existence of the few contributed lower neurons "explain away" the other muted lower neurons on generating upper neurons.

In order to compensate for this, we describe an alternative probabilistic model for attention, in which the role of the upper and lower layers in the GMM formulation are reversed.
This new attention scheme requires all the generated data (lower layer neurons) to be explained by at least one Gaussian center (upper layer neurons).
Therefore, it guarantees the preservation of information for all lower layer neurons, as we prove in this paper.

%Our insight here is that a formulation in which the upper layer acts as the generated data while the lower layer acts as Gaussian centers aligns well with the purpose of a {\em decoding} mechanism, where a final upper layer ultimately generates outputs.
%However, this design does not correspond well to the purpose of an {\em encoding} mechanism, in which the bottom layer takes input data and the upper layer encodes the data for representation purposes.
%To address this mismatch, we describe in this paper a new attention mechanism, in which the role of the upper and lower layers in the GMM formulation are reversed: the lower layer represents the data, while the upper layer represents the Gaussian centers.
The MLE equation of the reversed GMM model leads to a simple attention update that is similar to the original one, except for the attention weight normalization.
The original Transformer attention scheme only normalizes the attention weights once for every upper-layer neuron.
By contrast, our new attention mechanism requires a two-step attention weight normalization procedure:
the first normalizes each lower-layer neuron, and the second normalizes each upper-layer neuron.
In the rest of this paper, we denote the original, upper normalized attention scheme as \UNA, and the new doubly-normalized attention scheme as \DNA.

%It is worth noting that our \emph{doubly-normalized} attention is not \emph{doubly-stochastic}.
%After applying \DNA, the attention weights of the lower layer neurons are not normalized, since the upper layer normalization in the second step of \DNA~\emph{denormalizes} the lower layer.
We also show that \DNA~updates correspond exactly to one iteration of the Sinkhorn algorithm~\citep{peyre19computational} in a constrained optimization problem.
As a result, iterating \DNA~until convergence results in a \emph{doubly-stochastic} attention matrix where the attention weights of all upper and lower neurons are normalized.
We also showed that \UNA~can be formulated in a similar constrained optimization problem, except that the optimization problem of \UNA~does not have the constraint which presents ``explaining away'' compared to \DNA.

Mathematically, we also formalize the concept of ``explaining away'' of a lower neuron by using the sum of its attention weights.
We prove that the attention weights sum of the lower neurons of \DNA~are lower bounded by 1/(sequence length), therefore completely avoid the ``explaining away'' effect of \UNA.

Last but not least, we formulate a hybrid attention scheme, \HNA, that dynamically combines both attention schemes, and can provide a handle on a {\em task-based} preference between \UNA~and \DNA, as resulting from the learning algorithm.
We perform empirical studies and obtain clear numerical improvements using \DNA~and \HNA~formulation in several well-known benchmarks, with minor computational overhead and negligible increase of model size.

\section{Transformer Attention and Gaussian Mixture Models}
\label{sec:gmm-view}
In this section, we review the Transformer self-attention mechanism and analyze how it relates to the Gaussian Mixture Model.

Assuming a sequence of length $S$, we first focus on the Transformer single-headed attention formulation involving two layers of neurons: the lower-layer neurons are the input representations denoted as $\xb_j$ at position $j \in \cbr{1, \ldots, S}$,
and the upper-layer neurons are the output representations denoted as $\yb_i$ at position $i \in \cbr{1, \ldots, S}$.
We assume both $\xb_j$ and $\yb_i$ are 1-d tensors of the same size $D$.

The self-attention mechanism first transforms the input representations $\xb_j$ to queries and keys by applying $\qb_j = \Qb \xb_j$ and $\kb_j = \Kb\xb_j$, where $\Qb$ and $\Kb$ are trainable transformation matrices of size $D\times D$.
The value of an upper-layer neuron $\yb_i$ is computed as the weighted sum over the lower-layer neurons $\xb_j$ followed by the value transformation $\Vb$ of size $D\times D$,
\begin{align}
  \yb_i &= \sum_j \pi_{ij} \Vb\xb_j,  \label{eq:transformer}\\
  \text{where,}\; \pi_{ij} &= \frac{\exp(\qb_i^\top \kb_j)}{\sum_j \exp(\qb_i^\top \kb_j)}.\nonumber
\end{align}
Since in this formulation the attention weights $\pi_{ij}$ are normalized for every upper layer neuron $i$ over the lower layer neurons $j$, we refer to this attention scheme as upper-normalized attention, \UNA.

\subsection{Relation to GMM}
\label{sec:upper_gmm}
The \UNA~scheme \eqref{eq:transformer} relates to a Gaussian mixture model (GMM) in the following way.
Let us use $\kb_j$ to denote the positions of the Gaussian cluster centers, and the cluster priors denoted as $\alpha_j$, satisfying $\sum_j \alpha_j = 1$.
The generated data position is denoted as $\qb_i$.
If we assume the variance of the Gaussian distributions to be equal to 1\footnote{These assumptions are only needed to interpret the vanilla Transformer attention using GMM. Relaxing these assumptions does not affect derivations, and will lead to different forms of attention. Moreover, since the projection matrix $\Qb$, $\Kb$ are learnable, one can absorb the covariance into $\Qb$ and $\Kb$ and reparameterize to a Gaussian with unit variance. }, then the log-likelihood of the GMM is:
\begin{align*}
  \sum_i \log p(\qb_i) &= \sum_i \log \rbr{\sum_j \alpha_j \Ncal(\qb_i | \kb_j, 1)}.
\end{align*}
We can compute the optimal $\qb_i$ by taking the derivative of $\qb_i$ and solve the following equation,
\begin{align*}
   0=&\frac{\partial}{\partial \qb_i} \sum_i \log p(\qb_i) \\
   %= \frac{\partial}{\partial \qb_i} \log p(\qb_i) \\
   %&= \frac{\sum_j \alpha_j \frac{\partial}{\partial \qb_i} \Ncal(\qb_i | \kb_j, 1)}{\sum_j \alpha_j \Ncal(\qb_i | \kb_j, 1)} \\
   =& \frac{\sum_j \alpha_j \Ncal(\qb_i | \kb_j, 1) \frac{\partial \log \Ncal(\qb_i | \kb_j, 1)}{\partial \qb_i}}{\sum_j \alpha_j \Ncal(\qb_i | \kb_j, 1)}.
\end{align*}
If we assume the cluster priors\footnote{The cluster prior $\alpha_j$ favors the neurons with larger $|\kb_j|$, which intuitively are the ones carrying more information.} as $\alpha_j \propto \exp(\frac{1}{2} \kb_j^\top \kb_j)$, we have
\begin{align}
    \pi_{ij} &\triangleq \frac{\alpha_j \Ncal(\qb_i | \kb_j, 1)}{\sum_j \alpha_j \Ncal(\qb_i | \kb_j, 1)} \nonumber\\
    % & = \frac{\alpha_j \exp(-\frac{1}{2} (\qb_i - \kb_j)^\top (\qb_i - \kb_j))}{\sum_j \alpha_j \exp(-\frac{1}{2} (\qb_i - \kb_j)^\top (\qb_i - \kb_j))}  \\
   % & = \frac{\alpha_j \exp(-\frac{1}{2} (\qb_i^\top \qb_i - 2 \qb_i^\top \kb_j + \kb_j^\top \kb_j))}{\sum_j \alpha_j \exp(-\frac{1}{2} (\qb_i^\top \qb_i - 2 \qb_i^\top \kb_j + \kb_j^\top \kb_j))}  \nonumber\\
    &= \frac{\alpha_j \exp(\qb_i^\top \kb_j -\frac{1}{2} \kb_j^\top \kb_j)}{\sum_j \alpha_j \exp(\qb_i^\top \kb_j -\frac{1}{2} \kb_j^\top \kb_j)}\nonumber\\
    &= \frac{\exp(\qb_i^\top \kb_j)}{\sum_j \exp(\qb_i^\top \kb_j)}. \label{eq:pi_lower}
\end{align}
Using the fact that $\sum_j \pi_{ij} = 1$ and $\frac{\partial \log \Ncal(\qb_i | \kb_j, 1)}{\partial \qb_i} = \kb_j - \qb_i$, we obtain a fixed-point equation:
\begin{align}
  \qb_i =& \sum_j \pi_{ij} \kb_j. \label{eq:lower_attention}
\end{align}
If we compare Eq.~\eqref{eq:lower_attention} with Eq.~\eqref{eq:transformer}, the Gaussian cluster centers $\kb_j$ play exactly the same role as the key representation $\kb_j$ of the lower-layer neurons in Eq.~\eqref{eq:transformer}. The data position $\qb_i$ in Eq.\eqref{eq:pi_lower} plays the same role as the query representation $\qb_i$ in Eq.~\eqref{eq:transformer}. By iterating the fixed-point equation \eqref{eq:lower_attention} for one iteration, the new data position $\qb_i^{new} = \sum_j \pi_{ij} \kb_j$ corresponds to the upper layer neuron $\yb_i$ in Eq.~\eqref{eq:transformer} after applying the transformation $\Vb \Kb^{-1}$.

Note that computing the most-likely data positions $\qb_i$ given the Gaussian centers is non-standard for probabilistic inference. A more natural way would be the MLE estimation for the Gaussian centers given the data. That is exactly what doubly-normalized attention corresponds to, as we will discuss in the next section.

\subsection{Multi-head attention}
The multi-head ($H$ heads) attention can be derived similarly.
The lower neurons $\xb_j$ are projected into $H$ heads with different $\qb_j^h = \Qb^h \xb_j$ and $\kb_j^h = \Kb^h \xb_j$ where $\Qb^h$ and $\Kb^h$ are transformation matrices of size $\frac{D}{H} \times D$.
This yields $H$ outputs {$\yb^h_i$},
\begin{align}
  \yb^h_i &= \sum_j \frac{\exp(\qb_i^{h\top} \kb_j^h)}{\sum_j \exp(\qb_i^{h\top} \kb_j^h)} \Vb^h \xb_j, \label{eq:multihead}
\end{align}
where $\Vb^h$ is the value transformation matrix of size $\frac{D}{H} \times D$.
Similar to \eqref{eq:transformer}, \eqref{eq:multihead} corresponds to a GMM followed by value transformations\footnote{Some special treatments are needed to handle the value transformation since $\Kb^h$ is no longer square matrices. See the details in the Appendix.}.
$H$-heads attention corresponds to $H$ GMMs followed by value transformations.
The final output is a concatenation of all $H$ heads: $\yb_i = \text{concat}({\yb^h_i})$.

\section{Doubly-normalized Attention}
As we have shown, in the original \UNA~scheme, the lower layer neuron representations correspond to the Gaussian centers, while the upper layer neuron representations correspond to the data generated from these centers. The maximization with respect to the data positions is unnatural.
%This flow of information aligns well with the goal of a decoding mechanism, in which the upper layer generates the decoded outputs.
%However, for an encoding mechanism, the information flows in the other direction: data comes from the lower layer, while the role of the upper layer is to encode the signal from the lower-layer.
%As such, the current design of the Transformer self-attention mechanism, in which the \UNA~scheme is applied universally for both the encoder and the decoder networks, does not align well with the information flow of the encoding process.
In addition, the formulation has an ``explaining away'' effect, because for a GMM, not all Gaussian centers (lower layer neurons) are required to contribute in generating output data (upper layer neurons).
As a result, the information of the centers that do not generate data is completely lost.
For tasks such as summarization, ``explaining away'' may be acceptable, while for other tasks such as visual question answering and language modeling, the attention mechanism may benefit from a more ``conservative'' formulation, with the upper layer preserving the neural information at all positions.

To this end, we propose to reverse the role of the upper and lower layers in the GMM, so that all the generated data (lower layer neurons) will be explained by at least one Gaussian center (upper layer neurons).
This results in a new \emph{doubly-normalized attention} scheme (\DNA) (the derivation will be given shortly):
\begin{align}
  \yb_i &= \sum_j  \frac{\xi_{ij}}{\sum_j \xi_{ij}} \Vb\xb_j,  \label{eq:doubly_transformer}\\
  \text{where,}\;  \xi_{ij} &= \frac{\exp(\qb_i^\top \kb_j)}{\sum_i \exp(\qb_i^\top \kb_j)}.\nonumber
\end{align}
Comparing \eqref{eq:transformer} with \eqref{eq:doubly_transformer}, the only difference between the two is the normalization process of the attention weights.
The \DNA~scheme applies two normalization steps: first for each lower layer neuron $j$ and then for each upper layer neuron $i$.

\subsection{Relation to GMM}
We present here the derivation of \eqref{eq:doubly_transformer} from a GMM.
When we reverse the role of the upper and lower layers, we use $\qb_i$ to denote the Gaussian centers and $\kb_j$ as the data generated by GMM.
The log-likelihood function of the GMM is:
\begin{align}
   \sum_j \log p(\kb_j) &= \sum_j \log \rbr{\sum_i \beta_i \Ncal(\kb_j | \qb_i, 1)}, \label{eq:gmm_upper}
\end{align}
where the priors $\beta_i$ satisfy $\sum_i \beta_i = 1$.
We take the gradient with respect to $\qb_i$,
\begin{align*}
  &\frac{\partial}{\partial \qb_i} \sum_j\log p(\kb_j) \\
  %=& \sum_j \frac{\beta_i \frac{\partial}{\partial \qb_i} \Ncal(\kb_j | \qb_i, 1)}{\sum_i \beta_i \Ncal(\kb_j | \qb_i, 1)} \\
  =& \sum_j \frac{\beta_i \Ncal(\kb_j | \qb_i, 1) \frac{\partial}{\partial \qb_i} \log \Ncal(\kb_j | \qb_i, 1)}{\sum_i \beta_i \Ncal(\kb_j | \qb_i, 1)}.
\end{align*}
Define
\begin{align}
  \xi_{ij} \triangleq& \frac{\beta_i \Ncal(\kb_j | \qb_i, 1)}{\sum_i \beta_i \Ncal(\kb_j | \qb_i, 1)}\nonumber\\
  =& \frac{\beta_i \exp(\qb_i^\top \kb_j -\frac{1}{2} \qb_i^\top \qb_i)}{\sum_i \beta_i \exp(\qb_i^\top \kb_j -\frac{1}{2} \qb_i^\top \qb_i)}  \label{eq:upper_xi}
\end{align}

At optimum $\frac{\partial}{\partial \qb_i} \sum_j \log p(\kb_j) = 0$, we have $0 = \sum_j \xi_{ij} (\qb_i - \kb_j)$, and therefore the fixed-point equation is,
\begin{align}
    \qb_i &= \sum_j \frac{\xi_{ij}}{\sum_j \xi_{ij}} \kb_j. \label{eq:doubly_attention}
\end{align}
By iterating the fixed-point equation \eqref{eq:doubly_attention} for one iteration and assuming $\beta_i \propto \exp(\frac{1}{2} \qb_i^\top \qb_i)$, then the new center position $\qb_i^{new} = \sum_j \frac{\xi_{ij}}{\sum_j \xi_{ij}} \kb_j$ is equivalent to the upper layer neuron $\yb_i$ of Eq.~\eqref{eq:doubly_transformer}, modulo a transformation matrix $\Vb \Kb^{-1}$.

Similar to the \UNA, it is also straightforward to extend the above derivations to the multi-head ($H$-heads) \DNA~scheme, where it would be $H$ GMMs followed by value transformations.

\subsection{Relation to Double Stochasticity}
It should be emphasized that our \emph{doubly-normalized} attention is not \emph{doubly-stochastic} (where the columns and rows of the attention matrix $\pi_{ij}$ all sum to 1). After applying \DNA, the attention weights of the lower layer neurons are not normalized, since the upper layer normalization in the second step of \DNA~\emph{denormalizes} the lower layer. However, as we show in the following, \emph{doubly-stochastic} attention can be achieved by applying the two normalization steps for multiple iterations until convergence.

Consider the following constrained optimization problem that characterizes $\pi_{ij}$,
\begin{align}
 &\min_{\pi} \; \sum_{ij} \pi_{ij} D(\qb_i, \kb_j)  + \pi_{ij} \log \pi_{ij}\nonumber\\
 &\text{s.t.} \;\; \sum_i \pi_{ij} = 1,\;\; \sum_j \pi_{ij} = 1.  \label{eq:opt_doubly}
\end{align}
This problem is well-known in the optimal transport literature. The classical iterative algorithm for finding the solution is called the Sinkhorn algorithm~\citep{peyre19computational}, which uses the initial condition $\pi_{ij}^0 = \exp(-D(\qb_i, \kb_j))$ and iterates
\begin{align}
\xi_{ij}^{t}=\frac{\pi_{ij}^{t-1}}{\sum_i \pi_{ij}^{t-1}}, \;\;
\pi_{ij}^{t}=\frac{\xi_{ij}^{t}}{\sum_j \xi_{ij}^{t}}. \label{eq:sinkhorn}
\end{align}
If we write $D(\qb_i, \kb_j) := -\qb_i^\top \kb_j$ then the doubly-normalized attention weights computed in Eq.~\eqref{eq:doubly_transformer} correspond exactly to the updates~\eqref{eq:sinkhorn} of the Sinkhorn algorithm for one iteration.
If more iterations are applied, the attention weights will eventually satisfy both constraints in \eqref{eq:opt_doubly}, and become \emph{doubly-stochastic}. One question is whether \DNA~could perform better with more iterations for the updates in Eq.~\eqref{eq:sinkhorn}. Empirically, we find that adding more update iterations increases computational time but does not improve performance.

Interestingly, the attention weights of the original \UNA~scheme can be obtained from a very similar constrained optimization except that the normalization constraint on the lower layer neurons $j$ is removed:
\begin{align}
 &\min_{\pi} \; \sum_{ij} \pi_{ij} D(\qb_i, \kb_j)  + \pi_{ij} \log \pi_{ij} \nonumber\\
 &\text{s.t.} \;\; \sum_j \pi_{ij} = 1.  \label{eq:opt_lower}
\end{align}
Introducing the Lagrange multipliers $\lambda_i$, this formulation is equivalent to optimizing the Lagrangian,
%\begin{align*}
%L(\pi_{ij}, \lambda_i) = \sum_{ij} \pi_{ij} D(\qb_i, \kb_j)  + \pi_{ij} \log %\pi_{ij} + \sum_i \lambda_i \rbr{\sum_j \pi_{ij} - 1}
%\end{align*}
whose gradient with respect to $\pi_{ij}$ gives
\begin{align*}
\frac{\partial L(\pi_{ij}, \lambda_i)}{\partial \pi_{ij}} = D(\qb_i, \kb_j)  + 1 + \log \pi_{ij} + \lambda_i,
\end{align*}
and leads to the same attention weights as in Eq.~\eqref{eq:transformer} when $D(\qb_i, \kb_j) := -\qb_i^\top \kb_j$.

Comparing the two constrained optimization problems in \eqref{eq:opt_lower} and \eqref{eq:opt_doubly}, the removal of the constraint in \eqref{eq:opt_lower} allows solutions in which a lower-layer neuron $j$ has an arbitrary contribution to the upper layer, causing the ``explaining-away'' effect.

\subsection{Relation to Capsule Networks}
\label{sec:capsule}
It is also worth noting that \DNA~is related to the EM routing algorithm in the capsule networks~\citep{hinton18matrix}.
In particular, the vote matrix $V_{ij}$ in~\citep{hinton18matrix} is similar to $\kb_j$ in Eq.~\eqref{eq:gmm_upper}; the new pose matrix $\mu_j$ in~\citep{hinton18matrix} is similar to $\qb_i$ in Eq.~\eqref{eq:gmm_upper}.
However, unlike CapsuleNet, there is no variance $\sigma_i^2$ and $\beta_i$ estimation in \DNA, as we find that estimating variance $\sigma_i^2$ significantly hurts the empirical performance of the \DNA~algorithm. In addition, we only iterate the fixed-point equation \eqref{eq:doubly_attention} for one iteration, as more iterations are computationally expensive and does not improve the performance.

\section{Doubly-Normalized Attention Avoids Explaining Away}
\label{sec:explain_away}
In this section, we formalize the definition of ``explaining-away'' and compare \UNA~and \DNA~theoretically and empirically with respect to the ``explaining-away'' phenomenon.
\begin{definition}
In an attention unit, a lower-layer neuron $j$ is considered $\epsilon$-``explained away'', if the sum of the attention weights over the upper layer neurons $\sum_i \pi_{ij}$ is less than $\epsilon$.
\end{definition}
We consider $\epsilon$ to be some small value (fixed at $10^{-8}$ in the rest of this paper).
For the original Transformer \UNA, the only constraint in \eqref{eq:opt_lower} is $\sum_j \pi_{ij} = 1$. It does not require all lower layer neurons to be attended by the upper layer.
Therefore, for a certain lower-layer neuron $j$, the total attention weights to the upper layer $\sum_i \pi_{ij}$ can be as low as 0 so that it is $\epsilon$-``explained away''.
%which means the contribution of the lower neuron $j$ is completely filtered out in the upper layer. %This ``explaining-away'' phenomenon can potentially have negative effects for some applications: if early layers filter out the contribution of certain inputs that could be useful at later stages, that information can no longer be recovered.
% \begin{figure}[hbt!]
% \centering
% \begin{subfigure}{.35\textwidth}
% \centering
%   \includegraphics[width=1.\linewidth]{figs/min_low_weights.pdf}
% \end{subfigure}
% \caption{The minimum of attention weight sum in \UNA~and \DNA. }
% \label{fig:min_low_weights}
% \end{figure}

In contrast, the \DNA~scheme attempts to optimize the objective with both lower and upper layer normalization constraints \eqref{eq:opt_doubly} by one iteration of the Sinkhorn algorithm.
It turns out that this is sufficient to avoid the ``explaining-away'' phenomenon.
The following theorem formalizes this fact by showing that each lower-layer neuron contributes with a total attention weight of at least $1/S$, where $S$ is the sequence length.
\begin{theorem}
For any lower-layer neuron $j$, the sum of the doubly-normalized attention weights over the upper layer neurons $\sum_i \pi_{ij} = \sum_i \frac{\xi_{ij}}{\sum_j \xi_{ij}}$ is lower bounded by $1/S$.
\end{theorem}
\begin{proof}
Since $\sum_i \xi_{ij} = 1$,
\begin{align*}
  &\sum_i \frac{\xi_{ij}}{\sum_j \xi_{ij}} \\
  \ge& \sum_i \frac{\xi_{ij}}{\max_i (\sum_j \xi_{ij})} = \frac{\sum_i \xi_{ij}}{\max_i (\sum_j \xi_{ij})} \\ %= \frac{1}{\max_i (\sum_j \xi_{ij})} \\
  \ge& \frac{1}{\sum_j \max_i (\xi_{ij})} \ge \frac{1}{S}
\end{align*}
\end{proof}
We illustrate the difference between the two attention schemes, and how different they behave in practice with respect to the ``explaining-away'' phenomenon, using the multi-view attention model (with a single-layer, single-head attention) described in the VQA experiments later.
Fig.~\ref{fig:hist_low_weights} shows the histogram distribution of $\log_e (\sum_i \pi_{ij})$ between \UNA~and \DNA.
As the graph indicates, a large proportion of the \UNA~attention weights-sum is $\epsilon$-''explained-away'' ($\log_e$ values < $-20$), meaning that the information of only a few of the lower neurons are passed to the upper layer.
In contrast, \DNA~preserves more information from all lower layer neurons, as indicated by their weights-sum log values (> $-\log_e S$, where $S=100$).

Finally, we would like to emphasize that \DNA~does not work against attention sparsity. It allows the attention map $\pi_{ij} = 0$ between any pairs of neurons. What it forbids is the 0 total “contribution” of any lower neuron $j$: $\sum_{i} \pi_{ij} = 0$. Therefore, our method is compatible with existing faster sparse attention structures such as \cite{parmar2018image}.
%Fig.~\ref{fig:min_low_weights} shows the minimum attention weight sum of the lower layer neurons,  achieved by both \UNA~and \DNA~during training.
%As predicted by our analysis, the \UNA~scheme has $\min_j (\sum_i \pi_{ij})$ close to 0 (meaning that it explains-away at least some of its inputs), while the \DNA~scheme maintains $\min_j (\sum_i \pi_{ij})$ above its lower-bound value of $1/L=0.01$.
\begin{figure}[hbt!]
\centering
\begin{subfigure}{.35\textwidth}
\centering
  \includegraphics[width=1.\linewidth]{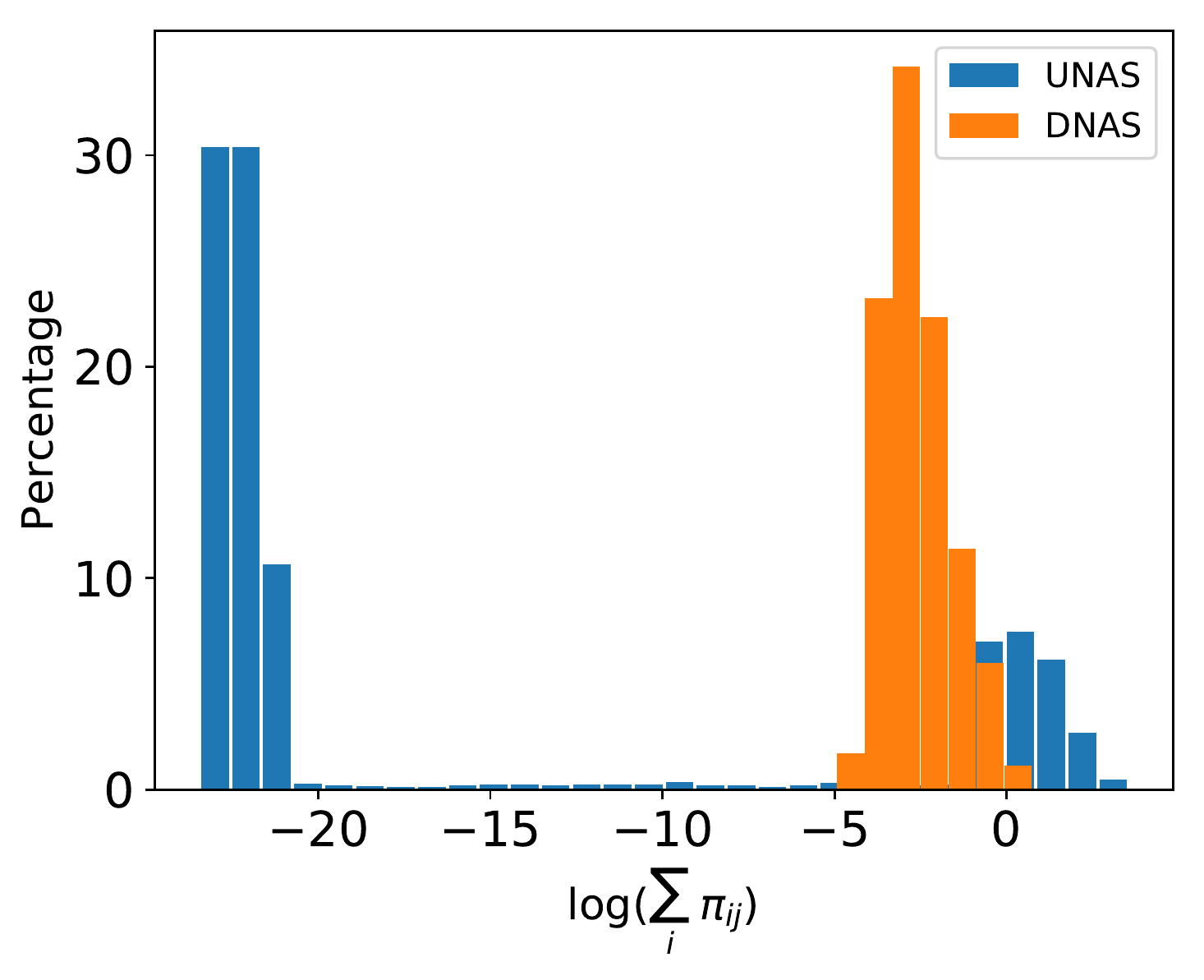}
\end{subfigure}\vspace{-10pt}
\caption{Comparison of the attention weights-sum between \UNA~and \DNA. Majority of the neurons in \UNA~are $\epsilon$-``explained away'', as the logarithm of their weights-sum is less than -20. }
\label{fig:hist_low_weights}
\end{figure}

\section{Hybrid Attention}
\label{sec:hybrid}
Since the formulations of \UNA~and \DNA~result in attention mechanisms with quite different properties, it is beneficial to combine them together.
A direct way to do so is by using trainable variables $u_l^h \in [0, 1]$ that control the contribution of the attention weights (for layer $l$ and head $h$) of the two normalization schemes (we use $u$ here to simplify the notation):
\begin{align}
    \pi_{ij} = u\; \pi^D_{ij} + (1-u) \pi^U_{ij}, \label{eq:hybrid}
\end{align}
where $\pi^D$ denotes the \DNA~weights and $\pi^U$ denotes the \UNA~weights.
We call this combination form the hybrid normalized attention scheme, \HNA.
\HNA~allows the model to learn, at different layers $l$ and different heads $h$, which of the two normalization schemes fits the data better, {\em for a given task.}
Each $u_l^h$ parameter is trained jointly with the other parameters to improve the representation power of the model and better fit the data.
Moreover, this approach also allows one to visualize how the values of the $u_l^h$ parameters change as the model is training, and therefore provides direct evidence of how much and where the different normalization schemes lead to better training performance.
We provide examples of such visualizations in the experiments.

\subsection{Computational Cost of \DNA~and \HNA}
The pseudo-code of the (multi-headed) \UNA, \DNA~and \HNA~is summarized in Algorithm \ref{alg:hnas}.
Note that for notational clarity, we wrote multi-head operations in a for-loop over different heads $h \in \cbr{1, \ldots, H}$.
However, an efficient implementation should use single tensor products across all heads, similar to the original Transformer method.
\begin{algorithm}
\caption{\textsc{\UNA, \DNA~and \HNA}}
\SetAlgoLined
\KwIn{Key, Query, Value transformation matrices $\Qb^h$, $\Kb^h$ and $\Vb^h$ for $H$ heads. Hybrid weights $u^h$ for all heads. Lower layer neurons $\xb$.}
\KwResult{Upper layer neurons $\yb$.}
\For{$h \in 1, \ldots, H$}{
  1. Compute $\qb_j^h = \Qb^h \xb_j$, $\kb_j^h = \Kb^h \xb_j$, $\vb_j^h = \Vb^h \xb_j$ for all lower neurons $j$ \\
  2. Compute $z^h_{ij} = \exp(\qb_i^{h\top} \kb_j^h)$\\
  3. [\UNA] Compute $\pi_{ij}^{h,U} = \frac{z_{ij}^h}{\sum_j z^h_{ij}}$\\
  4. [\DNA] Compute $\xi_{ij}^h = \frac{z_{ij}^h}{\sum_i z^h_{ij}}, \pi_{ij}^{h,D} = \frac{\xi_{ij}^h}{\sum_j \xi^h_{ij}}$ \\
  5. [\HNA] Compute $\pi_{ij}^{h} = u^h\; \pi^{h,D}_{ij} + (1-u^h) \pi^{h,U}_{ij}$\\
  6. Compute $\yb_i^h = \sum_j \pi_{ij}^{h} \vb_j^h$ \\
}
Return $\yb_i = \text{Concat}(\yb_i^h)$ for all $i$.
\label{alg:hnas}
\end{algorithm}

We can see that the additional computational cost of the \DNA~scheme compared to the original Transformer's \UNA~scheme is the two normalizations in Step-4 as opposed to one in Step-3. \HNA~requires both Step-3 and Step-4 and combines them together in Step-5.
The computational cost of the new steps is $O(S \times S \times H)$, where $S$ is the sequence length and $H$ is the number of heads.
In comparison, the cost of step 1 is $O(S \times D \times D)$, where $D$ is the size of the hidden representation.
In the majority of the applications we consider, we usually have $S \simeq D$ and $H \ll D$, and therefore the additional cost of the \DNA~and \HNA~scheme is usually small in practice.

The additional model variables introduced by the \HNA~scheme are the hybrid weights $u_l^h$.
Therefore, it adds $O(H\times L)$ new variables, where $L$ is the number of Transformer layers.
This increase is negligible compared to $O(D\times D\times L)$, the total size of the Transformer model.

\section{Numerical Experiments}
\label{sec:experiment}

%In this section, we contrast and numerically compare the original \UNA~scheme versus the proposed \HNA~scheme, for three end-to-end tasks: visual question answering, abstractive summarization and language representation learning.

\subsection{Multi-view Attention Model for VQA}
\label{sec:vqa}
%The goal of a visual question answering model is to provide an appropriate answer to a natural language question relevant to the contents of a given image.
In a vision-and-language multimodal system (e.g., Visual Question Answering), a crucial factor in the performance is the quality of the visual features.
A good example is the work of~\citep{yu19multimodal}, where they show that it is beneficial to use visual features produced by different image processing modules (multi-view).
They combine these visual features using an attention layer over the bounding-box features derived from multiple object detectors (Fig. \ref{fig:multi_view_vqa}).

\textbf{Experiment Setup.}
Our experimental setup is similar to the one proposed in \citep{yu19multimodal}.
We conduct experiments on the VQA benchmark dataset, VQA-v2 \citep{goyal2017making}.
%This benchmark contains a train set of 82,783 images and 443,757 questions, a validation set of 40,504 images and 214,354 questions, and test set consisting of 81,434 images and 447,793 questions.
Our core VQA model uses as a backbone the Pythia architecture~\citep{jiang2018pythia}.
%Aside from the backbone network, a crucial factor in the performance of a good VQA system is its visual feature extraction.
%All high-performing model use bounding-box visual features extracted by object-detector models trained on the Visual Genome Dataset~\citep{krishna2017visual}.
%Additionally, \cite{yu19multimodal} show that it is beneficial to use bounding-box features from multiple object detectors (multi-view).
We used three object detection models, where each detector generates $100$ bounding-box features.
All three object detection models are trained over the Visual Genome dataset~\citep{krishna2017visual}, but use different backbone networks: the first uses a ResNet-101 network~\citep{resnet2016}, the second a ResNet-200 network, and the third an Inception-ResNetV2 network~\citep{szegedy2016inceptionv4}.
\begin{figure}[hbt!]
\centering
\begin{subfigure}{.3\textwidth}
\centering
  \includegraphics[width=1.\linewidth]{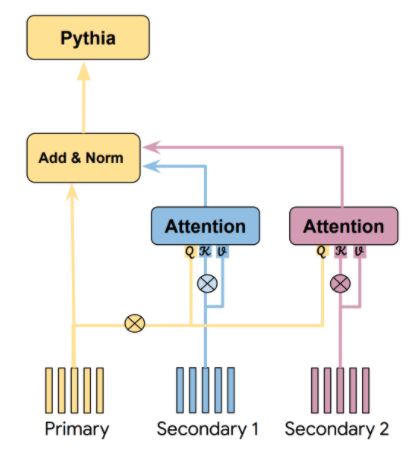}
\end{subfigure}\vspace{-10pt}
\caption{Multi-view attention model for VQA. }
\label{fig:multi_view_vqa}
\end{figure}

Multi-view features can be used in a straightforward manner by concatenating them all together before feeding them into the Pythia model; we call this approach the 3x100-boxes baseline.
The proposal from~\citep{yu19multimodal} combines the multi-view features using a one-layer attention model as follows:
one object-detector model is designated as primary, and its corresponding features are used as queries (after transformation);
the second and third object detection models are designated as secondary, and their corresponding features are used to obtain keys (see Figure \ref{fig:multi_view_vqa}).
The resulting output feature is a weighted sum of the features according to the attention weights.
More details about the mutliview attention model and the experiment hyperparameter settings are provided in the Appendix.
%With this attention model, the original $300$ bounding-box features from the three object-detection models are transformed into $100$ features, which are then fed into a Pythia model.
We use a single-layer and single-head attention model and experiment with two versions of the attention scheme: \UNA~and \DNA.

\begin{figure}[hbt!]
\centering
\begin{subfigure}{.35\textwidth}
\centering
  \includegraphics[width=1.\linewidth]{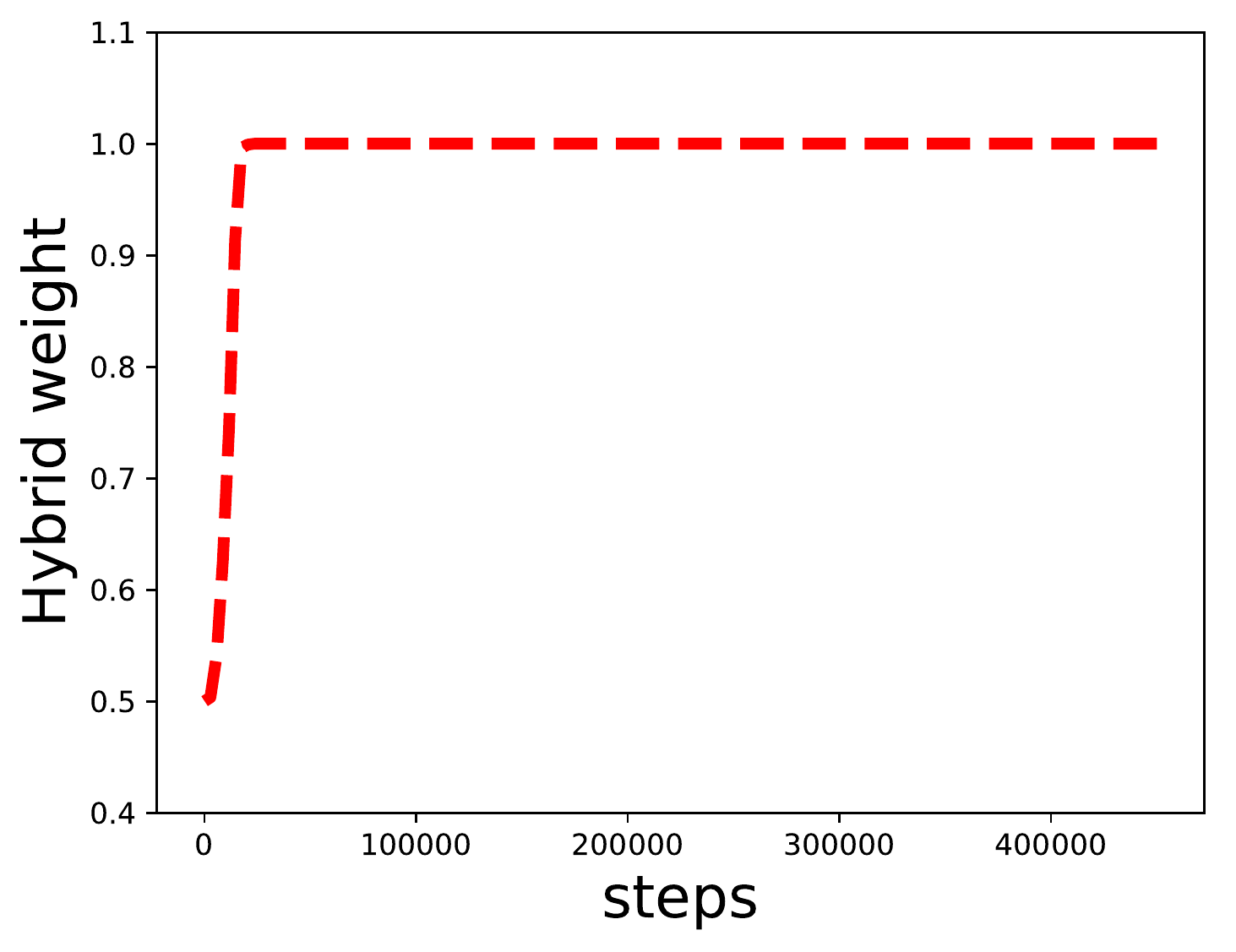}
\end{subfigure}
\caption{The hybrid weight heavily favors \DNA~over \UNA~in multi-view, attention-based VQA models. }
\label{fig:hybrid_weight_vqa}
\end{figure}

\begin{table*}[!htb]
  \begin{center}
    \begin{tabular}{ l |c c}
      Method &  Test-dev & Test-std \\ \hline
      10-100-boxes Pythia~\citep{jiang2018pythia} &   66.91 & - \\\hline
%      100-boxes Pythia~\citep{jiang2018pythia} &   68.31 & - \\\hline
%      100-boxes (no-attn baseline) &  68.33  & - \\
      3x100-boxes (no-attn baseline) &  68.79  & 69.22 \\
      3x100-boxes \UNA~(attn baseline) &    69.14  & 69.50 \\
      3x100-boxes \DNA &   {\bf 69.70}  & {\bf 70.01} \\ \hline
    \end{tabular}
  \end{center}
  \caption{Test Accuracy on VQA v2.0 Test-dev and Test-std splits.}
  \label{table:vqa}
\end{table*}

\textbf{Results Analysis.}
The results are summarized in Table \ref{table:vqa}.
%Confirming the findings from~\citep{yu19multimodal}, we see that using visual features from three object detectors improves performance over using the one from a single object detector (+0.46 on Test-dev).
Confirming the findings from~\citep{yu19multimodal}, using an attention mechanism (\UNA) over the 3x100 boxes improves the accuracy over the 3x100-boxes no-attn baseline, but the \DNA~mechanism achieves a better utilization of the signal provided by the three object detectors compared to the \UNA~mechanism. % (+0.56 on Test-dev and +0.51 on Test-std).
Moreover, \HNA~allows us to visually confirm the superiority of the \DNA~mechanism for the VQA task:
as we plot the hybrid weight $u$ from Eq.\eqref{eq:hybrid} in Fig.~\ref{fig:hybrid_weight_vqa}, it rapidly converges to 1.0, meaning that the model learns to heavily favor \DNA~over \UNA~for combining multi-view features.
Combining the findings in Fig.~\ref{fig:hist_low_weights}, we believe that \UNA~performs worse because it $\epsilon$-``explains-away'' too many box features in this stage, while \DNA~preserves information from all bounding boxes.

\subsection{Language Representation Learning}
The goal of language representation learning is to pretrain textual representations that are useful for solving natural language understanding (NLU) tasks like entailment or question answering.
\begin{figure}[hbt!]
\centering
\begin{subfigure}{.35\textwidth}
\centering
  \includegraphics[width=1.\linewidth]{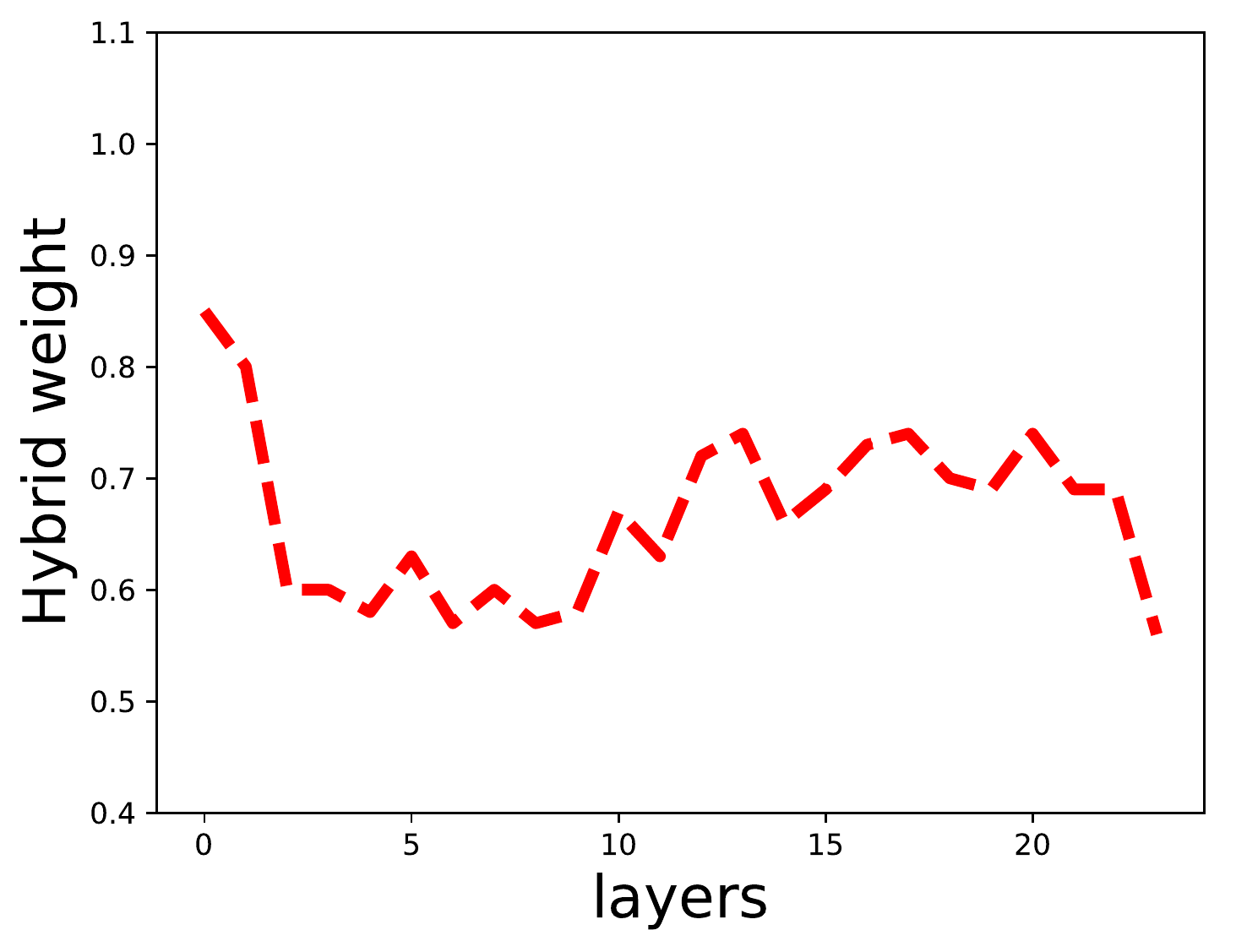}
\end{subfigure}
  \caption{In the BERT model, the hybrid weights favor \DNA~in all layers of the encoder ($u\geq .5$);
    \UNA~gains more weight for closer-to-output layers.}
\label{fig:hybrid_weight_bert}
\end{figure}

\textbf{Experiment Setup.} We use the BERT \citep{devlin2018bert} setting for our language representation learning setup:
a Transformer network with 24 layers of attention, the hidden and embedding size set to 1024, and 16 attention heads.

\begin{table*}[!htb]
  \begin{center}
    \begin{tabular}{ l | c c c c c}
      Method & \squad1.1 (EM/F1) & \squad2.0 (EM/F1)  & RACE & GLUE (avg.)\\ \hline
            \UNA~(baseline) & 85.1$\pm0.2$/92.2$\pm0.2$ & 80.2$\pm0.1$/83.6$\pm0.1$ & 74.2$\pm0.2$ & 84.5$\pm 0.3$ \\
            \DNA & {\bf 85.8}$\pm 0.1$/{\bf 92.4}$\pm 0.0$ & 81.0$\pm0.2$/84.2$\pm 0.2$ & {\bf 74.3$\pm0.3$} & {\bf 85.2$\pm 0.2$} \\
      \HNA & 85.6$\pm0.1$/92.2$\pm0.1$ & {\bf 81.7$\pm0.1$}/{\bf 84.8$\pm0.1$}  & {\bf 74.3$\pm0.2$} &  84.7$\pm 0.3$\\ \hline
    \end{tabular}
  \end{center}
  \caption{Pretraining with BERT models and finetuning on several representative downstream tasks.}
  \label{table:bert}
\end{table*}
\begin{table*}[!htb]
  \begin{center}
    \begin{tabular}{ l |c c c }
      Method & ROUGE-1 & ROUGE-2 & ROUGE-L \\ \hline
      %ABS+~\citep{rush2015neural}                 & 29.78          & 11.89          & 26.97 \\
      %Feat2s~\citep{nallapati2016abstractive}     & 32.67           & 15.59          & 30.64 \\
      %SEASS~\citep{zhou2017selective}             & 36.15           & 17.54          & 33.63 \\
      %Base+E2Tcnn+sd~\citep{amplayo2018entity}    & 37.04           & 16.66          & 34.93 \\
%      Transformer~\citep{goodman2019multistage} & 38.05          & 18.95          & 35.26 \\\hline
%      Transformer $\RandInit$~\citep{goodman2019multistage} & 38.05          & 18.95          & 35.26 \\\hline
%      Transformer $\BertInit$~\citep{goodman2019multistage} & 38.96          & 19.55          & 36.22 \\\hline
      \UNA-encoder, \UNA-decoder (baseline)  & 38.02$\pm$0.07 & 18.93$\pm$0.10 & 35.25$\pm$0.09 \\
      \DNA-encoder, \UNA-decoder & 38.19$\pm$0.05 & 19.09$\pm$0.07 & 35.52$\pm$0.06 \\
      \HNA-encoder, \UNA-decoder & {\bf 38.27$\pm0.12$} & {\bf 19.30$\pm0.07$} & {\bf 35.56$\pm0.09$} \\ \hline
    \end{tabular}
  \end{center}
  \caption{ROUGE F1 scores for headline generation on the Gigaword benchmark.}
  \label{table:headline}
\end{table*}

Our experiment is based on the ALBERT platform~\citep{lan2019albert}\footnote{\url{https://github.com/google-research/albert/}}.
We use the \textsc{BookCorpus}~\citep{zhu2015aligning} and English Wikipedia~\citep{devlin2018bert} to pretrain three contextual representation models, using \UNA, \DNA, and \HNA~respectively.
Each pretraining uses a batch size of 4096 and a LAMB optimizer with learning rate
0.00176 for 125k steps on the Cloud TPU V3 with 64 TPUs. We evaluate the resulting representations by using them as a starting point to finetune for a number of representative NLU tasks~\citep{rajpurkar-etal-2018-know,williams-etal-2018-broad}.
Due to space limitation, more experimental details are provided in the Appendix.

\textbf{Results Analysis.}
Each fine-tuning experiment is done 5 times, and the mean number and their standard error are reported.
The main results are summarized in Table \ref{table:bert} and more detailed results are available in the Appendix.
Overall, the network parameters encode their language representations by making use of \DNA, resulting in the empirical advantage of the \DNA~and \HNA~based models over the \UNA~based models on most tasks considered.
Aside from the numerical improvements when finetuning on the task, we also inspect what happens to the hybrid weight $u$ of Eq.\eqref{eq:hybrid} during \HNA~pretraining.
In Fig.~\ref{fig:hybrid_weight_bert}, we plot the hybrid weights (averaged over all heads of each layer) for all 24 layers and find that they are always larger than 0.5, meaning that the \DNA~method is preferred for pretraining (masked-LM \& sentence-ordering) tasks.
The \UNA~method has more weight for higher layers, meaning that ``explaining away'' is more allowable when it is closer to the output.

\subsection{Headline Generation}
\label{sec:summ}
We also present empirical results on a summarization task.
As already mentioned, summarization aligns well with the tendency of \UNA~of ``explaining away'' unimportant information.

\textbf{Experiment Setup.}
We use the Gigaword dataset~\citep{ldc-english-gigaword}, which is a standard benchmark for headline generation.
We pre-process this dataset as in~\citep{rush2015neural}, and further tokenize the words into word-pieces~\citep{devlin2018bert}, which results in a vocabulary size of 30,522 word-piece types.
We use a 10k dataset for validation, and the standard 2k test set~\citep{rush2015neural} as the evaluation test.

Our model and training hyperparameters are adapted from~\citep{goodman2019multistage}.
The transformer  contains 12 layers, each with a hidden size of 768 and 12 attention heads.
We keep the attention mechanism in the decoder as \UNA, and compare the \DNA~and \HNA~with \UNA~as the encoder attention mechanism.
%Since \DNA~may change the output of previously decoded positions, it is inappropriate for sequential decoding. Therefore, in encoder-decoder models, our decoder is always based on \UNA, while we compare \HNA~and \UNA~only on encoders.
Our training uses a batch size of 512 and an Adam optimizer~\citep{Kingma2015AdamAM} with learning rate of $2e^{-5}$ for 500k steps.
The training is done on Cloud TPU V3 with 16 TPUs for each job.
\begin{figure}[hbt!]
\centering
\begin{subfigure}{.35\textwidth}
\centering
  \includegraphics[width=1.\linewidth]{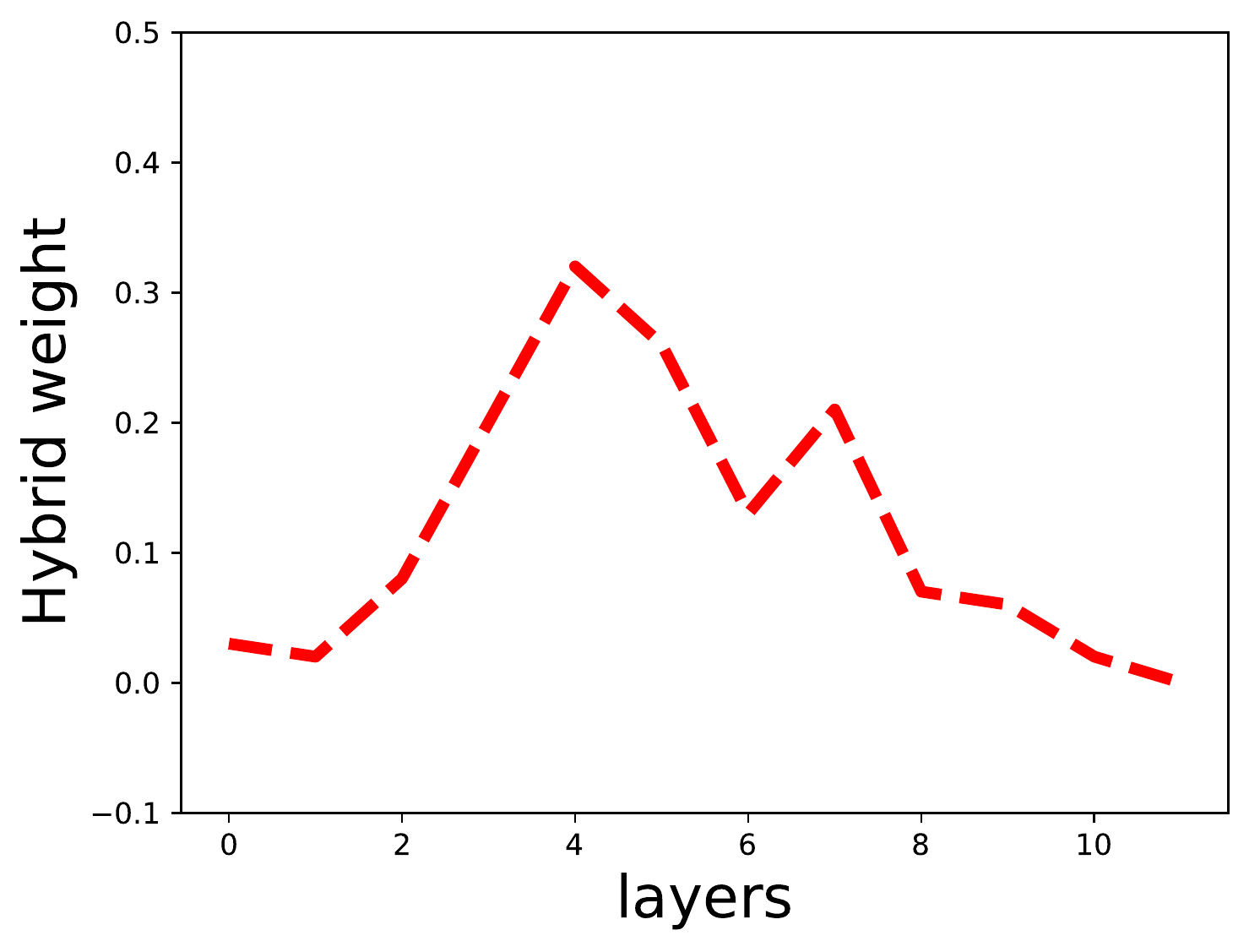}
\end{subfigure}
\caption{The hybrid weights favor \UNA~in the encoder of headline generation, because the task requires filtering unimportant information. However, the ROUGE scores of \DNA~is higher than \UNA.}
\label{fig:hybrid_weight_headline}
\end{figure}

\textbf{Results Analysis.} Each experiment is run 5 times, and the mean number and standard error are reported in Table \ref{table:headline}.
We also plot the averaged hybrid weights for all layers in Fig.~\ref{fig:hybrid_weight_headline} which shows that the \HNA~model favors \UNA, especially in the top and bottom layers of the encoder.
Nevertheless, \DNA~still makes a positive contribution in the middle layers, which allows the model based on \HNA~to perform better compared to the \UNA-based one. Somewhat surprisingly, \DNA~alone performs competitvely: all of its ROUGE scores are higher than the ones of \UNA~and are close to the ones of \HNA. This indicates that complete "explaining away" by \UNA~is unnecessary for filtering unimportant information. \DNA~provides a conservative alternative which achieves better generation performance.

\section{Conclusion}
The formulation of the attention mechanism of the Transformer, here called \UNA, leads to ``explaining away'' effects in which the information of certain input neurons is completely ignored.
Our new \DNA~scheme compensates for \UNA's weaknesses by avoiding ``explaining away'', as we show both theoretically and empirically.
Empirically, we show \DNA~and a hybrid \HNA~to be superior to the original attention mechanism, at the cost of minor computational overhead.

\bibliography{garcon}
\bibliographystyle{acl_natbib}

\newpage
\clearpage

\appendix

\section{Multi-head attention and GMM}
In multi-head attention, the lower neurons $\xb_j$ are projected into $H$ heads with different $\qb_j^h = \Qb^h \xb_j$ and $\kb_j^h = \Kb^h \xb_j$ where $\Qb^h$ and $\Kb^h$ are transformation matrices of size $\frac{D}{H} \times D$.
This yields $H$ outputs {$\yb^h_i$},
\begin{align*}
  \yb^h_i &= \sum_j \frac{\exp(\qb_i^{h\top} \kb_j^h)}{\sum_j \exp(\qb_i^{h\top} \kb_j^h)} \Vb^h \xb_j,
\end{align*}
where $\Vb^h$ is the value transformation matrix of size $\frac{D}{H} \times D$.

If we follow the same idea as in single-head attention, the corresponding GMM becomes,
\begin{align*}
  \qb_i^{h,new} =& \sum_j \frac{\exp(\qb_i^{h\top} \kb_j^h)}{\sum_j \exp(\qb_i^{h\top} \kb_j^h)} \Kb^h \xb_j
\end{align*}
In order to convert $\qb_i^{h,new}$ to $\yb^h_i$, one difficulty is that $\Kb^h$ is a down-projection matrix. Therefore, the inversion of $\Kb^{h \top} \Kb^h$ does not exist. In order to avoid the problem, one can use the same key transformation for all heads $\bar{\Kb}^h = \Kb$ which is $D \times D$. The query transformation $\bar{\Qb}^h$ is a zero padded matrix also of size $D \times D$. The rows of $\bar{\Qb}^h$ are all zero, except of the rows
\begin{align*}
\bar{\Qb}^h\sbr{\rbr{\frac{hD}{H}:\frac{(h+1)D}{H}}, :} = \Qb^h.
\end{align*}
One can show that, if $\bar{\qb}_j^h = \bar{\Qb}^h \xb_j$ and $\bar{\kb}_j^h = \bar{\Kb}^h \xb_j$, then
\begin{align*}
\bar{\qb}_i^{h\top} \bar{\kb}_j^h = \qb_i^{h\top} \kb_j^h.
\end{align*}
Therefore, the corresponding GMM becomes
\begin{align*}
  \bar{\qb}_i^{h,new} =& \sum_j \frac{\exp(\bar{\qb}_i^{h\top} \bar{\kb}_j^h)}{\sum_j \exp(\bar{\qb}_i^{h\top} \bar{\kb}_j^h)} \bar{\Kb}^h \xb_j
\end{align*}
and can be related to $\yb^h_i$ by $\yb^h_i = \Vb^h \Kb^{-1} \bar{\qb}_i^{h,new}$.

\section{Experiment Details about the Multi-view Attention Model for VQA}
\paragraph{Dataset and evaluation} The VQA-v2 \cite{goyal2017making} dataset contains a training set (with 80k images and 444k QA pairs), a validation set (with 40k images and 214k QA pairs), and test set (with 80k images and 448k QA pairs). For each question, there are $10$ answers provided by $10$ different human annotators. Following the same setting as Pythia~\cite{jiang2018pythia}, we augment the train set with a part of validation set (train + val2train) and use the remaining data in validation set for validation (minival). The test set is split into test-dev and test-std, and the evaluation can only be conducted online. Same as other work on VQA, we report a robust accuracy metric
as the average score over $9$ subsets of the groundtruth $10$ answers, where each score is computed as follows:
$$\text{Acc}(ans) = \min \{{(\# \text{human that said } ans)}/{3},1\}.$$

\paragraph{Detailed Model Descriptions}

Our VQA model uses as a backbone the Pythia architecture~\citep{jiang2018pythia}. In order to combine the 100 features from each of the three object detection models, we use a one-layer attention mechanism as in \citep{yu19multimodal}. The features from one object-detector model is used as the primary feature. The features of the second and third object detection models are designated as secondary features. In order to obtain keys and queries, we apply transformation on the secondary and primary features, so that $\kb_i^{S1} = \Kb^{S1} \xb_i^{S1}$, $\kb_i^{S2} = \Kb^{S2} \xb_i^{S2}$, $\qb_i = \Qb \xb_i^{P}$. However, we find that it is better to directly use the features as the values without transformation. For the primary view, the output value of the $i$-th feature is
\begin{align*}
    \yb_i^{P} = \xb_i^{P}.
\end{align*}
For each secondary view, the feature is computed as
\begin{align*}
    \yb_i^{S1} &= \sum_j \pi_{ij}^{S1} \xb_j^{S1}
\end{align*}
For the \UNA~scheme,
\begin{align*}
\pi_{ij}^{S1} &= \frac{\exp(\qb_i^{\top} \kb_j^{S1})}{\sum_j \exp(\qb_i^{\top} \kb_j^{S1})}.
\end{align*}
For the \HNA~schme
\begin{align*}
\pi_{ij}^{S1, U} &= \frac{\exp(\qb_i^{\top} \kb_j^{S1})}{\sum_j \exp(\qb_i^{\top} \kb_j^{S1})}\\
\xi_{ij}^{S1} &= \frac{\exp(\qb_i^{\top} \kb_j^{S1})}{\sum_i \exp(\qb_i^{\top} \kb_j^{S1})}, \;\;\;\;
\pi_{ij}^{S1, D} = \frac{\xi_{ij}^{S1}}{\sum_j \xi_{ij}^{S1}} \\
\pi_{ij}^{S1} &= u \pi_{ij}^{S1, D} + (1 - u) \pi_{ij}^{S1, U}.
\end{align*}
The final output feature integrates the 100 features from different views via an element-wise summation, followed by layer normalization,
\begin{align*}
 \yb_i = normalize(\yb_i^P + \yb_i^{S1} + \yb_i^{S2})
\end{align*}

\paragraph{Hyperparameters} During the hyperparameter tuning process, we train on training set only and manually tune our hyperparameter based on the accuracy on the validation set. We use the same model hyperparameters as the Pythia model. Our image feature dimension is $2048$ and the query and key transformation matrices are of size $2048 \times 2048$. For the attention layer, we experiment with multiple number of heads including 1, 2, 4, and 8, and we find the single head attention gives the best performance. We also did a grid search on the dropout probability in attention layer from $0.05$ to $0.4$, and set it to $0.1$ after the search. The hybrid attention weight is initalized to be 0.5. For optimization, we use Adam optimizer with learning rate $10^{-4}$, and use batch size $192$. We train the model for $500,000$ steps. The training was done on 4 Cloud TPUs. The total training time is approximately 38 hours for each model. The validation performance on the minival dataset is reported in Table \ref{table:vqa_val}.

\begin{table*}[!htb]
  \begin{center}
    \begin{tabular}{ l |c }
      Method &  minival \\ \hline
      3x100-boxes (no-attn baseline) &  68.26  \\
      3x100-boxes \UNA~(attn baseline) &    68.34   \\
      3x100-boxes \HNA &   {\bf 68.99}   \\ \hline
    \end{tabular}
  \end{center}
  \caption{Validation accuracy on the VQA v2.0 minival splits.}
  \label{table:vqa_val}
\end{table*}

\section{Experiment Details about Language Representation Learning}
\subsection{Downstream Evaluation Tasks}
\label{downstream_detailed_description}
\paragraph{\squad} \squad is an extractive question answering dataset built from Wikipedia. The answers are segments from the context paragraphs and the task is to predict answer spans. We evaluate our models on two versions of SQuAD: v1.1 and v2.0. \squad v1.1 has 100,000 human-annotated question/answer pairs. \squad v2.0 additionally introduced 50,000 unanswerable questions. For \squad v1.1, we use the same training procedure as BERT, whereas for \squad v2.0, models are jointly trained with a span extraction loss and an additional classifier for predicting answerability~\citep{yang2019xlnet,liu2019roberta}. We report the results on the development set.

\paragraph{RACE} RACE is a large-scale dataset for multi-choice reading comprehension, collected from English examinations in China with nearly 100,000 questions. Each instance in RACE has 4 candidate answers. Following prior work~\citep{yang2019xlnet,liu2019roberta}, we use the concatenation of the passage, question, and each candidate answer as the input to models. Then, we use the representations from the ``[CLS]'' token for predicting the probability of each answer. The dataset consists of two domains: middle school and high school. We train our models on both domains and report accuracies on the development set.

\paragraph{GLUE} GLUE~\citep{williams-etal-2018-broad} is comprised of 9 tasks, namely Corpus of Linguistic Acceptability (CoLA), Stanford Sentiment Treebank (SST), Microsoft Research Paraphrase Corpus
(MRPC), Semantic Textual Similarity Benchmark (STS),
Quora Question Pairs (QQP), Multi-Genre NLI (MNLI), Question NLI (QNLI), Recognizing Textual
Entailment (RTE) and
Winograd NLI (WNLI). It focuses on evaluating model capabilities for natural language understanding. The detailed per-task results on GLUE are available in Table \ref{table:glue}.

\subsection{Model hyperparameters}
Our pretraining uses the same default hyperparameters as in \url{https://github.com/google-research/albert/blob/master/run_pretraining.py}.
The total number of model parameters of the BERT model is about 334M. The total pretraining time for \UNA~is about 40 hours per job, while for \DNA~and \HNA~are around 48 hours. There is about 20\% overhead which is much higher than our theoretical estimation. This is because our BERT pretraining used 64 TPUs that are highly efficient for parallelizing large matmul ops. As a result, the runtime of two consecutive normalization steps of smaller tensors could be longer than a single-step matmul of a much larger tensor. We expect the relative overhead to be smaller with other types of processing units.

Hyperparameters for downstream tasks are shown in Table~\ref{tab:downstream-hyperparameter}.
These hyperparameters were copied from \cite{lan2019albert} which were adapted from \cite{liu2019roberta},  \cite{devlin2018bert}, and \cite{yang2019xlnet}.
We used the ADAM optimizer for fine-tuning as in \cite{lan2019albert}.
\begin{table*}[!htbp]
    \small
    \centering
\begin{tabular}{c|ccccccc}
& LR & BSZ & BERT DR & Classifier DR & TS & WS & MSL \\\hline
SQuAD v1.1 & 5.00E-05 & 48 & 0 & 0.1 & 3649 & 365 & 384 \\
SQuAD v2.0 & 3.00E-05 & 48 & 0 & 0.1 & 8144 & 814 & 512 \\
RACE & 1.00E-05 & 32 & 0 & 0.1 & 12000 & 1000 & 512 \\
CoLA & 1.00E-05 & 16 & 0 & 0.1 & 5336 & 320 & 512 \\
STS & 2.00E-05 & 16 & 0 & 0.1 & 3598 & 214 & 512 \\
SST-2 & 1.00E-05 & 32 & 0 & 0.1 & 20935 & 1256 & 512 \\
MNLI & 3.00E-05 & 128 & 0 & 0.1 & 10000 & 1000 & 512 \\
QNLI & 1.00E-05 & 32 & 0 & 0.1 & 33112 & 1986 & 512 \\
QQP & 5.00E-05 & 128 & 0.1 & 0.1 & 14000 & 1000 & 512 \\
RTE & 3.00E-05 & 32 & 0.1 & 0.1 & 800 & 200 & 512 \\
MRPC & 2.00E-05 & 32 & 0 & 0.1 & 800 & 200 & 512 \\
WNLI & 2.00E-05 & 16 & 0.1 & 0.1 & 2000 & 250 & 512 \\
\end{tabular}
    \caption{Hyperparameters for language representation learning downstream tasks. LR: Learning Rate. BSZ: Batch Size. DR: Dropout Rate. TS: Training Steps. WS: Warmup Steps. MSL: Maximum Sequence Length.}
    \label{tab:downstream-hyperparameter}
\end{table*}

\begin{table*}[!htb]
  \begin{center}
    \begin{tabular}{ l | c c c c c c c c | c}
      Method & MNLI & SST-2 & CoLA & QNLI & QQP	& RTE & STS-B & MRPC & Avg\\ \hline
      \UNA & 85.5$\pm$.3 & 93.1$\pm$.2 & 60.7$\pm$.6 & 91.1$\pm$.1 & 89.4$\pm$.8 & 76.2$\pm$.5 & 91.1$\pm$.1 & 88.7$\pm$.1 & 84.5$\pm$.3 \\
      \DNA & 86.4$\pm$.1 & 93.1$\pm$.1 & 59.9$\pm$.7 & 91.5$\pm$.1 & 91.2$\pm$.1 & 80.3$\pm$.6 & 91.1$\pm$.1 & 87.7$\pm$.2 & 85.2$\pm$.2 \\
      \HNA & 86.2$\pm$.1 & 93.2$\pm$.1 & 59.4$\pm$.5 & 91.4$\pm$.1 & 91.1$\pm$.1 & 77.8$\pm$1.0 & 90.8$\pm$.1 & 87.6$\pm$.3 & 84.7$\pm$.3 \\
    \end{tabular}
  \end{center}
  \caption{Detailed results of \HNA~and \DNA~on GLUE downstream tasks.}
  \label{table:glue}
\end{table*}

\section{Experimental Details about Headline Generation}
The Gigaword dataset~\citep{ldc-english-gigaword} consists of about 4M $\langle\mathit{article}, \mathit{headline}\rangle$ pairs.
We pre-process this dataset as in~\citep{rush2015neural}, which results in an average $\mathit{article}$ length of 31.4 words, and an average $\mathit{headline}$ length of 8.5 words.
We further tokenize the words into word-pieces~\citep{devlin2018bert}, which results in a vocabulary size of 30,522 word-piece types.
We use a 10k dataset for validation, and the standard 2k test set~\citep{rush2015neural} as the evaluation test.

Our backbone Transformer model is adapted from~\citep{goodman2019multistage} that contains 12 layers, each with a hidden size of 768 and 12 attention heads. The total number of model parameters is about 108M.
We truncate (or pad) the input and output sequences to a fixed number of word-piece positions, namely 128 encoder positions and 64 decoder positions, to accommodate hardware and model-architecture limitations.
The hybrid attention weight is initalized to be 0.1, because the headline generation task favors \UNA~to "explain away" unimportant neurons.
We use an Adam optimizer~\citep{Kingma2015AdamAM} and a learning rate of $2e^{-5}$ for 500 steps.
The training was done on Cloud TPU V3 with 16 TPUs for each job.
The total training time is approximately 16.5 hours for \DNA/\HNA~and 16 hours for \UNA.

The ROUGE-L score on the validation set is 45.75 for \HNA, and 45.63 for \UNA.

\section{Doubly-normalized Attention Alleviates Mode Collapse}
Attention model tends to collapse modes. In particular, the data at different positions tend to move closer to each other after attention. We illustrate the collapsing effect in a 2-D example in Fig.~\ref{fig:mode_collapse_2d_unbalance}, where two separated clusters of data converge to a single point after only 4 steps of \UNA~(left). Most multi-layer attention models such as the Transformer try to avoid such collapsing effect by adding a residual layer, which pulls the data back to its original position.

\begin{figure}[hbt!]
\begin{subfigure}{.23\textwidth}
  \centering
  \includegraphics[width=1\linewidth]{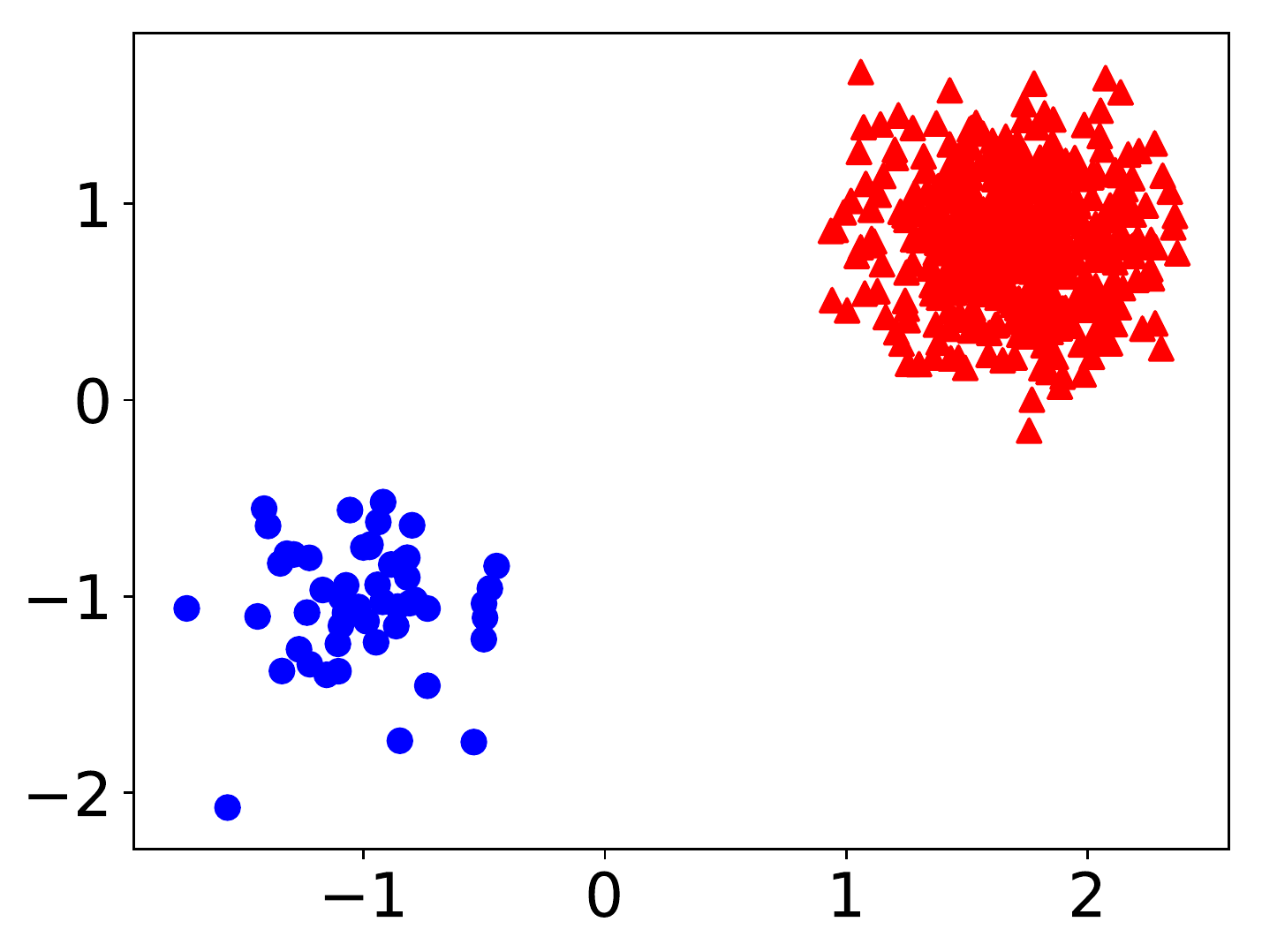}
  \label{fig:sfig1}\vspace{-5mm}
  \caption{\UNA, step 0}
\end{subfigure}%
\begin{subfigure}{.23\textwidth}
  \centering
  \includegraphics[width=1\linewidth]{figs/ub_data.pdf}
  \label{fig:sfig1}\vspace{-5mm}
  \caption{\DNA, step 0}
\end{subfigure}%

\begin{subfigure}{.23\textwidth}
  \centering
  \includegraphics[width=1\linewidth]{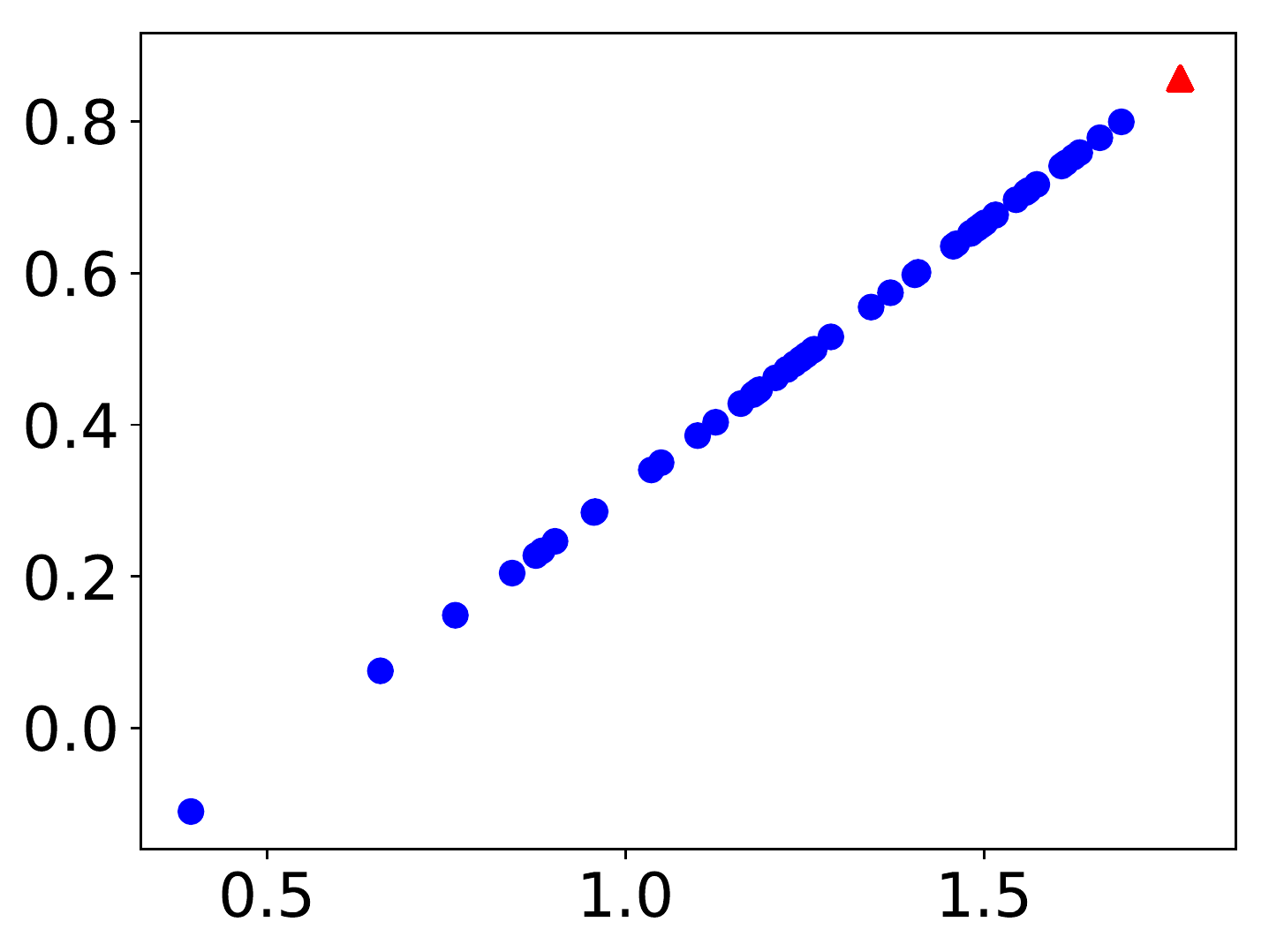}
  \label{fig:sfig2}\vspace{-5mm}
  \caption{\UNA, step 2}
\end{subfigure}
\begin{subfigure}{.23\textwidth}
  \centering
  \includegraphics[width=1\linewidth]{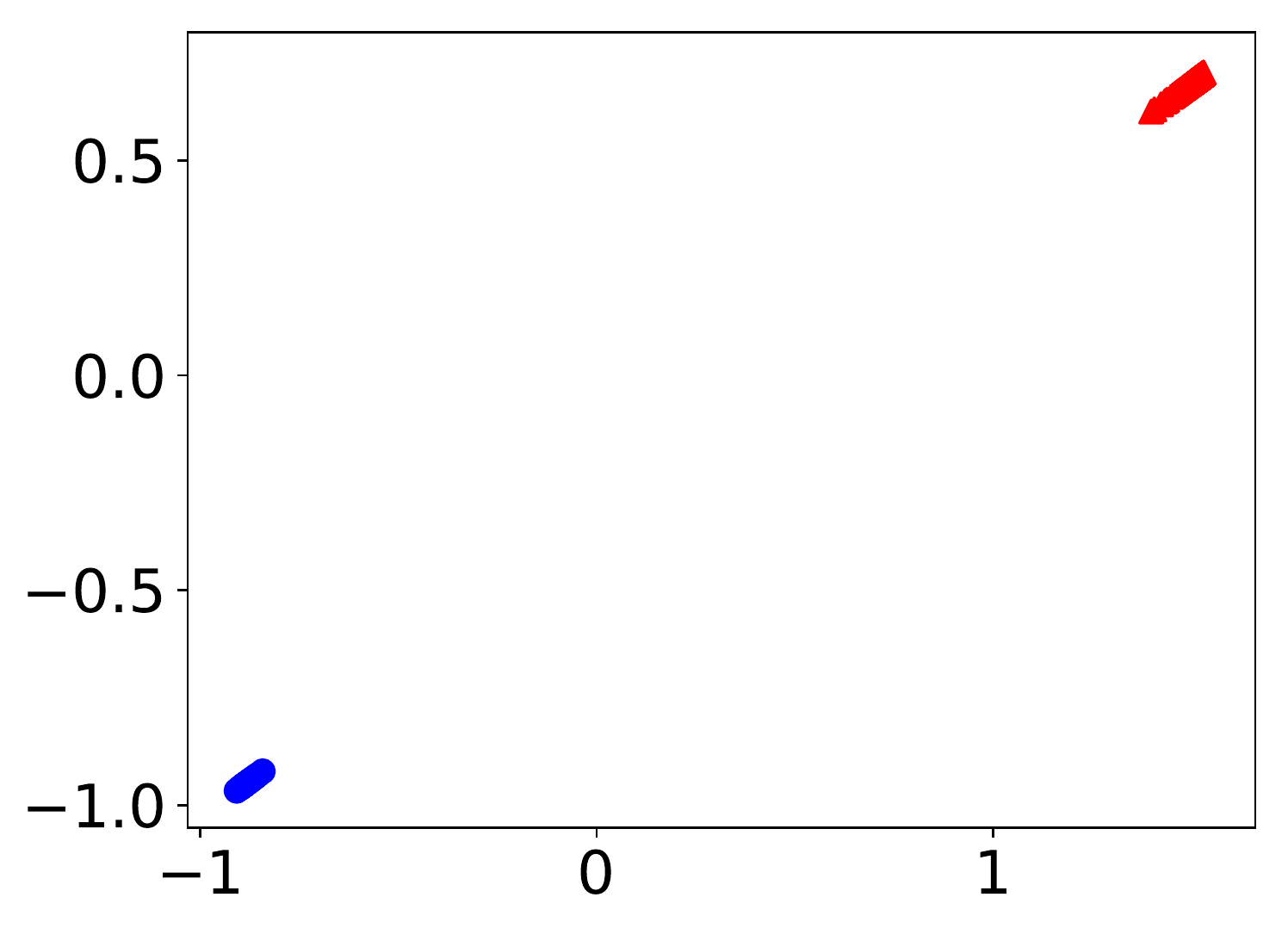}
  \label{fig:sfig2}\vspace{-5mm}
  \caption{\DNA, step 2}
\end{subfigure}

\begin{subfigure}{.23\textwidth}
  \centering
  \includegraphics[width=1\linewidth]{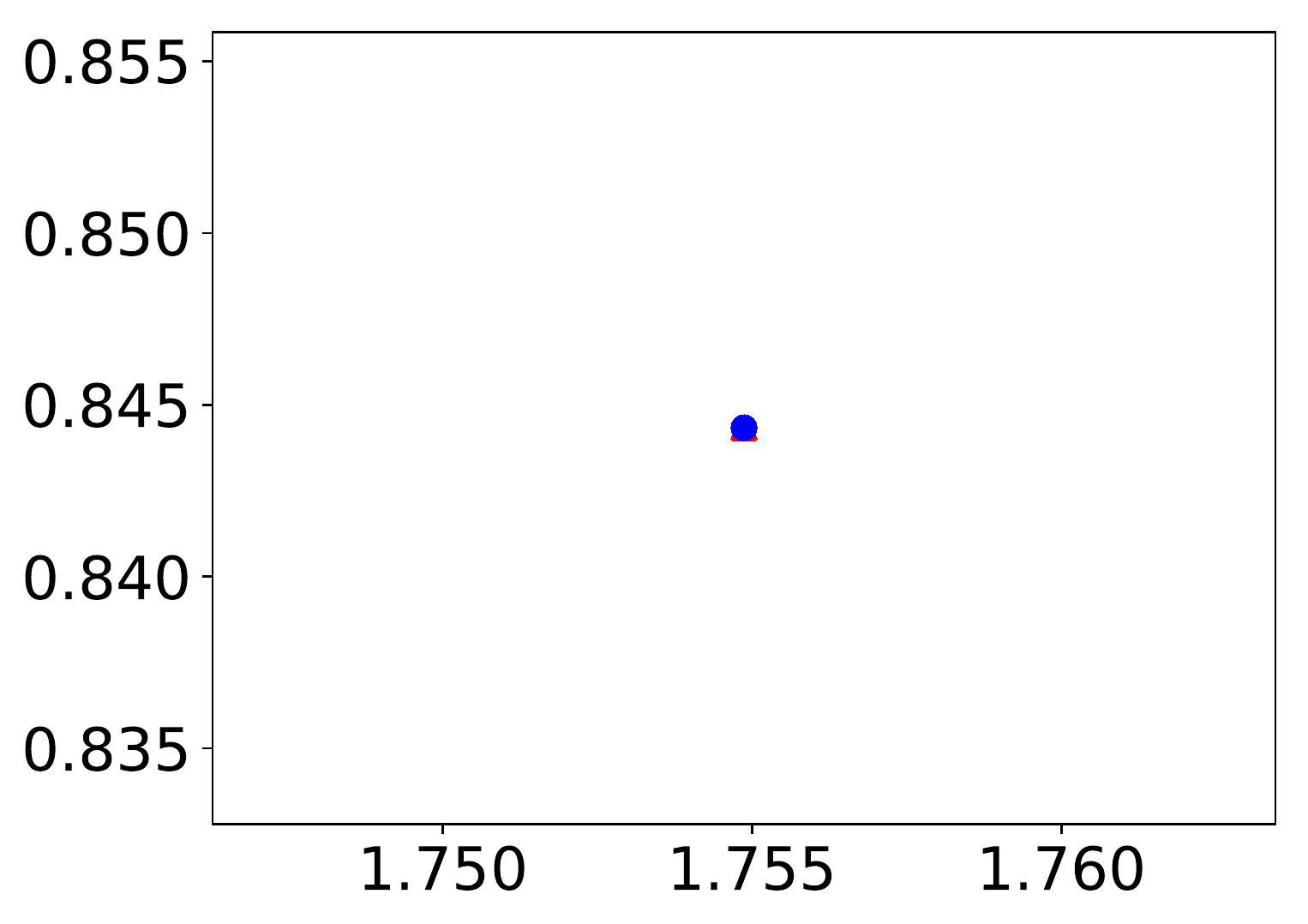}
  \label{fig:sfig1}\vspace{-5mm}
  \caption{\UNA, step 4}
\end{subfigure}%
\begin{subfigure}{.23\textwidth}
  \centering
  \includegraphics[width=1\linewidth]{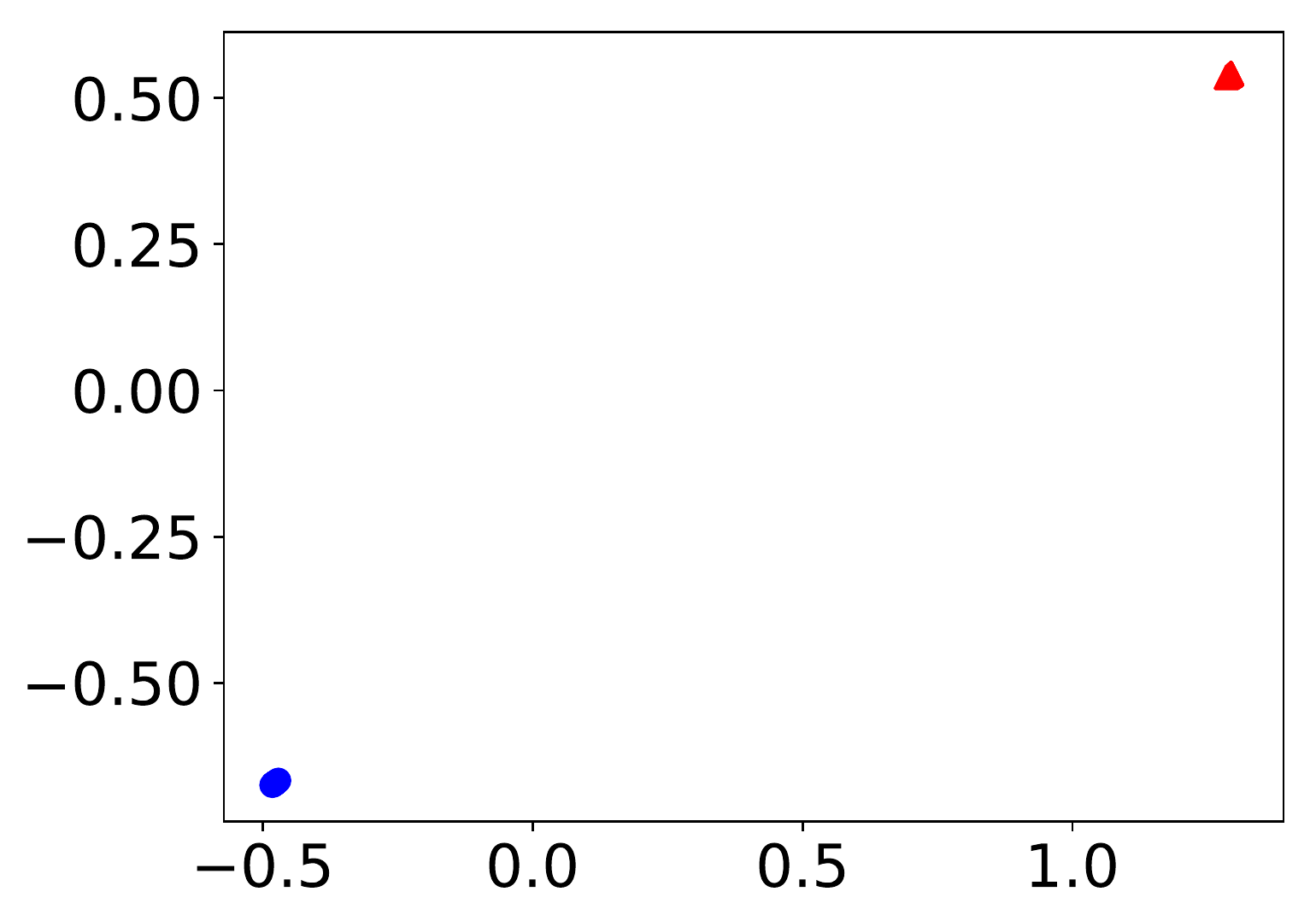}
  \label{fig:sfig1}\vspace{-5mm}
  \caption{\DNA, step 4}
\end{subfigure}%
\vspace{-2mm}
\caption{Mode-collapsing behavior on a mixture of two Gaussians. 500 data points (red) are centered at [1.8, 0.7] and the other 50 data points (blue) are centered at [-1, -1]. Both Gaussians have covariance matrix equal to $0.1 \Ib$. Four steps of self-attention are applied on the data points. In each step, Eq.\eqref{eq:lower_attention} is applied by \UNA~and Eq.\eqref{eq:upper_xi} and \eqref{eq:doubly_attention} are applied by \DNA, and we let $\Qb=\Kb=\Ib$ in both cases. After four steps, \UNA~(left) collapses to 1 cluster, while \DNA~(right) maintains 2 clusters.}
\label{fig:mode_collapse_2d_unbalance}
\end{figure}

To compare the mode-collapsing effect of \UNA~and \DNA~analytically, we study a 1-D toy example which contains two clusters. One cluster contains $N_0$ data points centered at value $a$, and the other contains $N_1$ data points centered at value $-a$.
The distance between the two centers is $2a$.
Assuming the relative distance between the data points within each set is negligible compared to $2a$, the unnormalized attention weights between one center and the data from the other set is $s = \exp(-(2a)^2/2) = \exp(-2a^2)$,
and the weights between one center and the data within that set is
$t = \exp(0)=1$\footnote{The attention weights are computed with a Gaussian.
But the same result holds with dot product attention, where the inter-attention weight is $s = \exp(\inner{-a}{a})=\exp(-a^2)$ and the intra-attention weight is $t = \exp(\inner{a}{a})=\exp(a^2)$.
The ratio $s/t = \exp(-2a^2)$ is identical to the Gaussian case.}. We compare the center distance between the two data clusters after applying the \UNA~and \DNA~self-attention updates.

Applying Eq.~\eqref{eq:lower_attention} for the \UNA~scheme, the new center distance of the upper-normalized attention scheme are:
\begin{align*}
    c_0^U &= \rbr{\frac{N_0 t}{N_0 t + N_1 s} - \frac{N_1 s}{ N_0 t + N_1 s}} a \\
    &= \frac{N_0 t-N_1 s}{N_0 t + N_1 s} a \\
    c_1^U &= \rbr{\frac{N_0 t}{N_0 t + N_1 s} - \frac{N_1 s}{ N_0 t + N_1 s}} a \\
    &= \frac{N_0 s-N_1 t}{N_1 t + N_0 s} a
\end{align*}
and the distance between the two updated centers is:
\begin{align*}
   c_0^U - c_1^U = \frac{2 N_0 N_1 (t^2 - s^2)a}{(N_1 t + N_0 s)(N_0 t + N_1 s)}.
\end{align*}
Since we have that $t = 1$, defining $r = N_0 / N_1$ then gives
\begin{align}
    c_0^U - c_1^U = \frac{2 r(1-s^2)a}{(1 + r s) (r + s)}. \label{eq:lower_mode_dist}
\end{align}

By contrast, if we apply the Eq.~\eqref{eq:doubly_attention} updates for the \DNA~scheme, the new center position of the doubly-normalized attention scheme are:
\begin{align*}
    &c_0^D\\
    =& \frac{\frac{N_0 ta}{N_0 t + N_1 s}}{\frac{N_0 t}{N_0 t + N_1 s} + \frac{N_1 s}{N_0 s + N_1 t}} - \frac{\frac{N_1 sa}{ N_0 s + N_1 t}}{\frac{N_0 t}{ N_0 t + N_1 s} + \frac{N_1 s}{ N_0 s + N_1 t}} \\
    =& \frac{N_0 t (N_0 s + N_1 t) - N_1 s (N_0 t + N_1 s)}{N_0 t (N_0 s + N_1 t) + N_1 s (N_0 t + N_1 s)} a, \\
    &c_1^D\\
    =& \frac{\frac{N_0 sa}{N_0 t + N_1 s}}{\frac{N_0 s}{N_0 t + N_1 s} + \frac{N_1 t}{N_0 s + N_1 t}} - \frac{\frac{N_1 ta}{ N_0 s + N_1 t}}{\frac{N_0 s}{ N_0 t + N_1 s} + \frac{N_1 t}{ N_0 s + N_1 t}} \\
    =& \frac{N_0 s (N_0 s + N_1 t) - N_1 t (N_0 t + N_1 s)}{N_0 s (N_0 s + N_1 t) + N_1 t (N_0 t + N_1 s)} a,
\end{align*}
and the distance between the two updated centers is:
\begin{align*}
   & c_0^D - c_1^D \\
   =& 2 N_1 a\{\frac{t(N_0 t + N_1 s)}{N_0 s (N_0 s + N_1 t) + N_1 t(N_0 t + N_1 s)} \\
   &- \frac{s(N_0 t + N_1 s)}{N_0 t (N_0 s + N_1 t) + N_1 s(N_0 t + N_1 s)}\}.
\end{align*}
Since again $t = 1$, defining $r = N_0 / N_1$ and $q = \frac{N_0 t + N_1 s}{N_0 s + N_1 t} = \frac{r + s}{rs + 1}$
then yields
\begin{align}
    c_0^D - c_1^D = \frac{2qr(1-s^2)a}{(q+rs)(r+sq)}. \label{eq:upper_mode_dist}
\end{align}
We plot the values of Eq.~\eqref{eq:lower_mode_dist} and Eq.~\eqref{eq:upper_mode_dist} on the $y$-axis against that of $r = N_0 / N_1$ on the $x$-axis, for several different $a$ values, see Fig.~\ref{fig:cnt_dist_1d}.
We see that in both cases the distance between the two centers decays after the attention updates.
However, the center distance of \DNA~ always upper bounds the one of \UNA, with the gap getting larger as the cluster sizes get more unbalanced ($r \neq 1$). The above result holds for the 2-D example in Fig.~\ref{fig:mode_collapse_2d_unbalance} as well, where the \UNA~collapses to a single cluster after 4 steps (left) while the \DNA~ maintains two separate clusters (right).

\begin{figure}[hbt!]
\centering
\begin{subfigure}{.22\textwidth}
  \centering
  \includegraphics[width=1\linewidth]{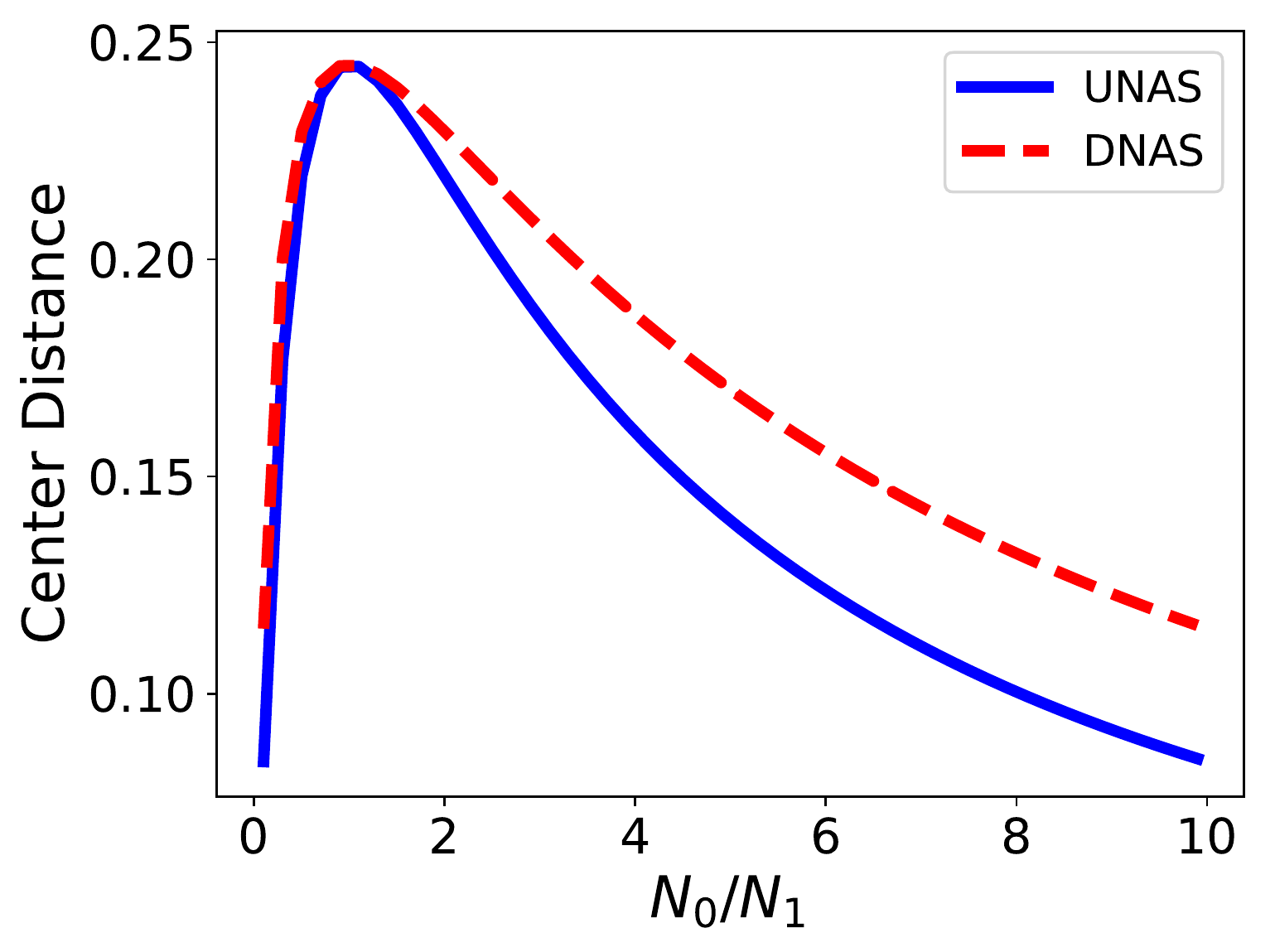}
  \caption{$a=0.5$}
\end{subfigure}
\begin{subfigure}{.22\textwidth}
  \centering
  \includegraphics[width=1\linewidth]{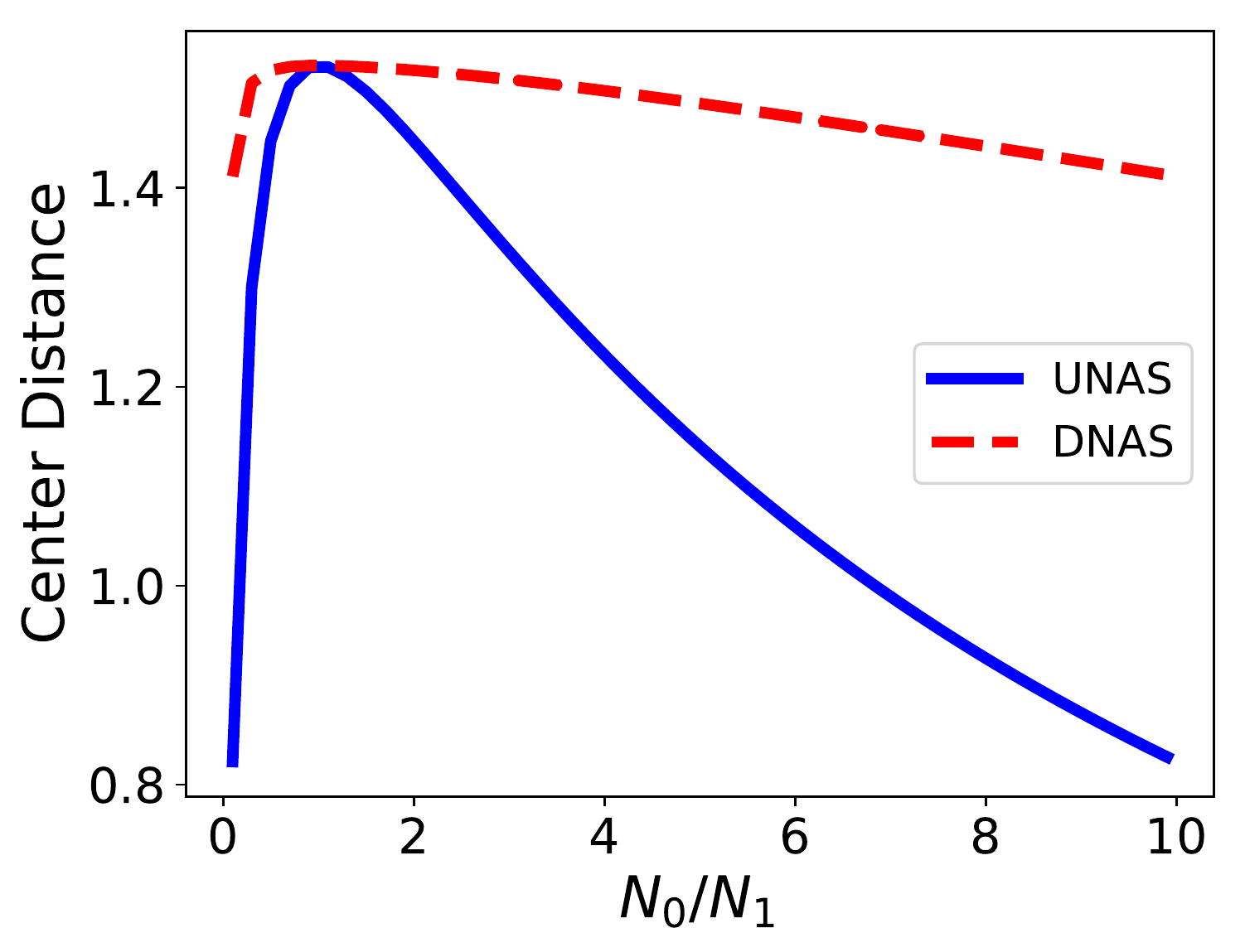}
  \caption{$a=1$}
\end{subfigure}
%\begin{subfigure}{.22\textwidth}
%  \centering
%  \includegraphics[width=1\linewidth]{figs/cnt_dist_2.pdf}
%  \caption{$a=2$}
%\end{subfigure}

\caption{Center distance values after \UNA~(blue solid curve) and \DNA~(red dashed curve), as a function of cluster mass ratio $r=N_0/N_1$ with different $a$ values (initial distance between centers is $2a$).}
\label{fig:cnt_dist_1d}
\end{figure}

The mode collapse effect is even more obvious in multi-layer attention.
In Fig.~\ref{fig:mode_collapse_2d_balance}, when the two clusters are balanced (both clusters contain 225 data points), both normalization schemes yield similar results.
However, when the two clusters are unbalanced (the red cluster contains 500 points and the blue one contains 50) (Fig.~\ref{fig:mode_collapse_2d_unbalance2}), \UNA~collapses to a single cluster after 4 steps, while the \DNA~maintains two separate clusters.

\begin{figure}[hbt!]
\begin{subfigure}{.23\textwidth}
  \centering
  \includegraphics[width=1\linewidth]{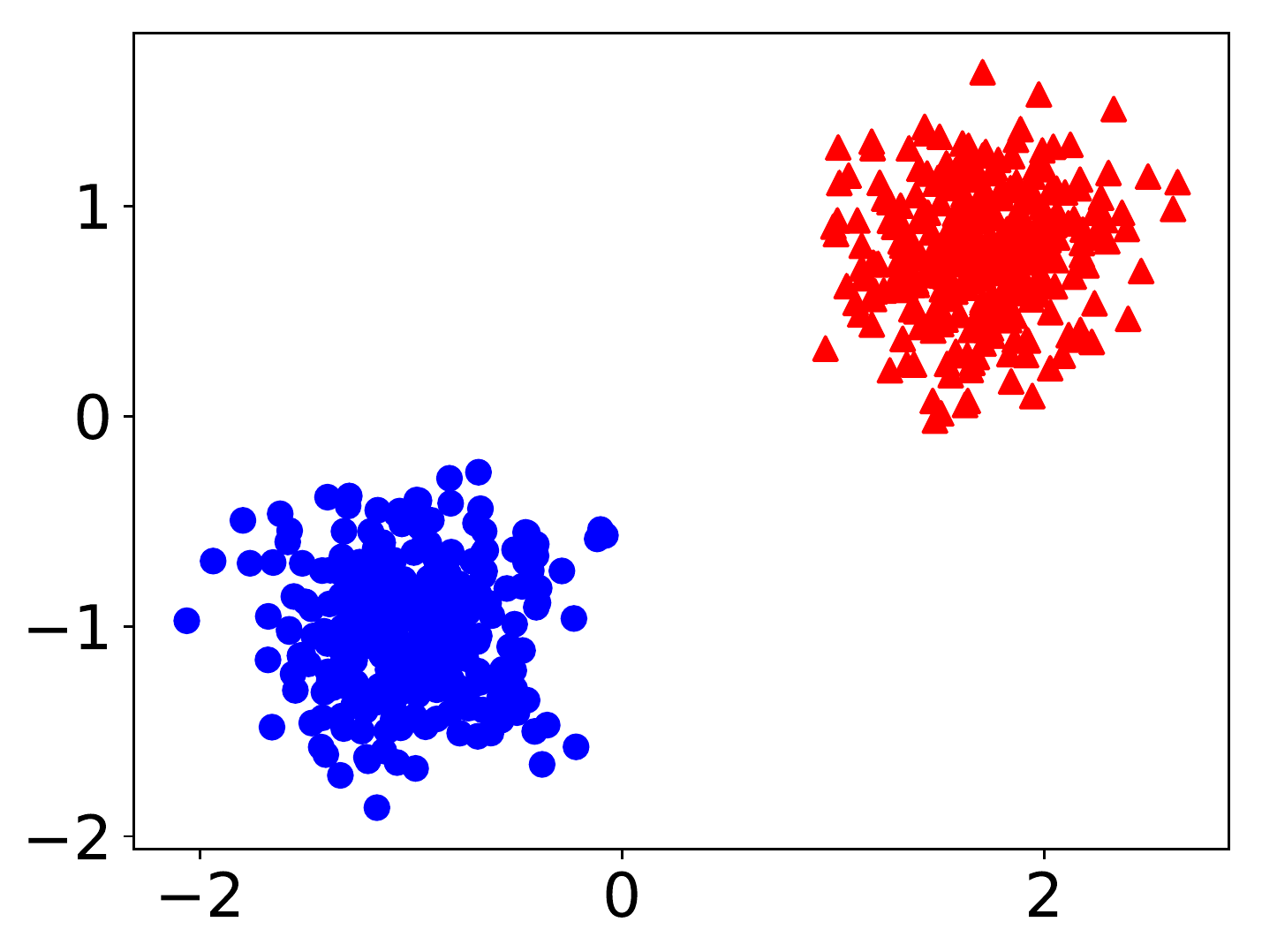}
  \label{fig:sfig1}\vspace{-5mm}
  \caption{\UNA, step 0}
\end{subfigure}
\begin{subfigure}{.23\textwidth}
  \centering
  \includegraphics[width=1\linewidth]{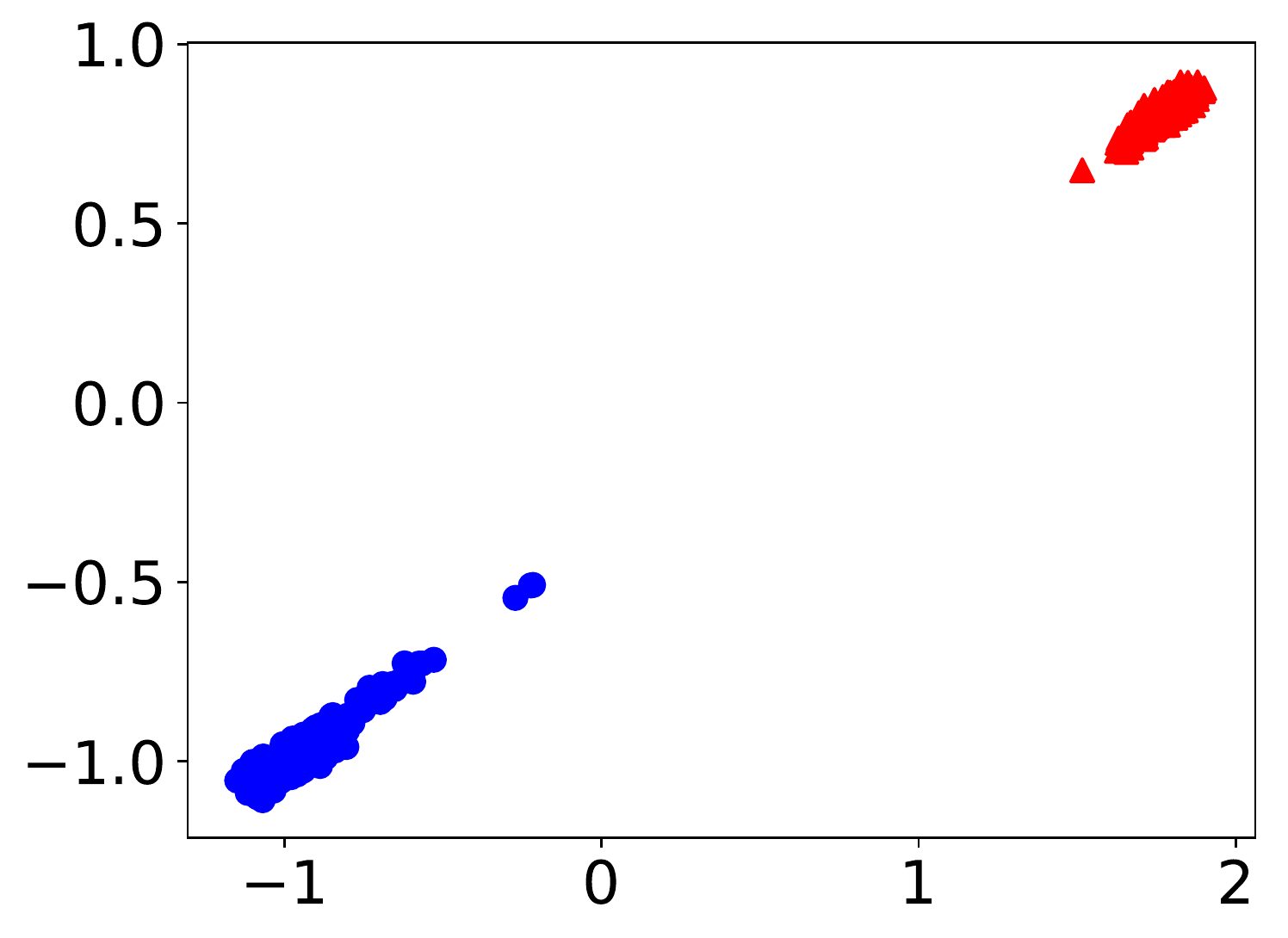}
  \label{fig:sfig2}\vspace{-5mm}
  \caption{\UNA, step 1}
\end{subfigure}
\begin{subfigure}{.23\textwidth}
  \centering
  \includegraphics[width=1\linewidth]{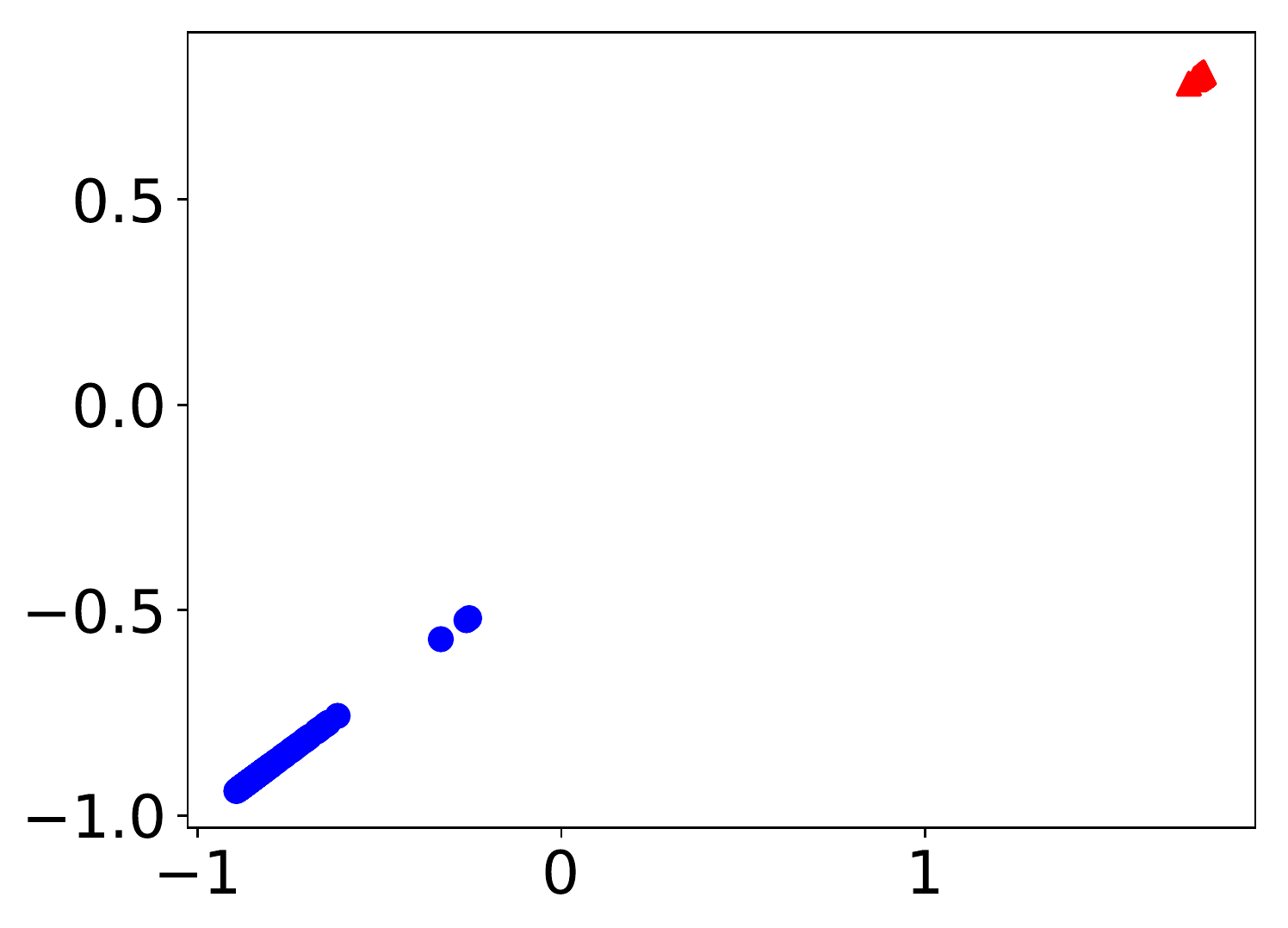}
  \label{fig:sfig2}\vspace{-5mm}
  \caption{\UNA, step 2}
\end{subfigure}
\begin{subfigure}{.23\textwidth}
  \centering
  \includegraphics[width=1\linewidth]{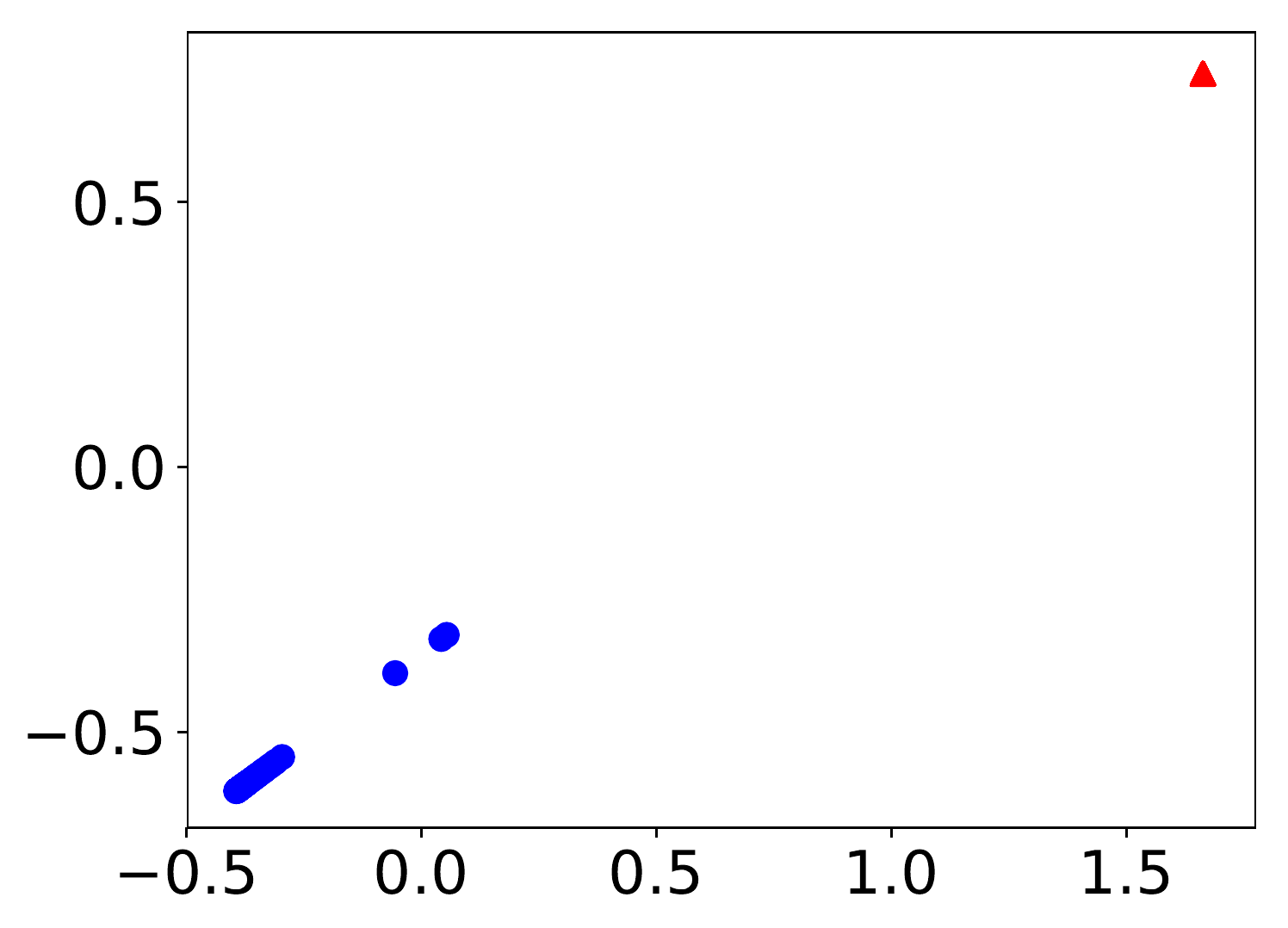}
  \label{fig:sfig1}\vspace{-5mm}
  \caption{\UNA, step 4}
\end{subfigure}

\begin{subfigure}{.23\textwidth}
  \centering
  \includegraphics[width=1\linewidth]{figs/b_data.pdf}
  \label{fig:sfig1}\vspace{-5mm}
  \caption{\DNA, step 0}
\end{subfigure}%
\begin{subfigure}{.23\textwidth}
  \centering
  \includegraphics[width=1\linewidth]{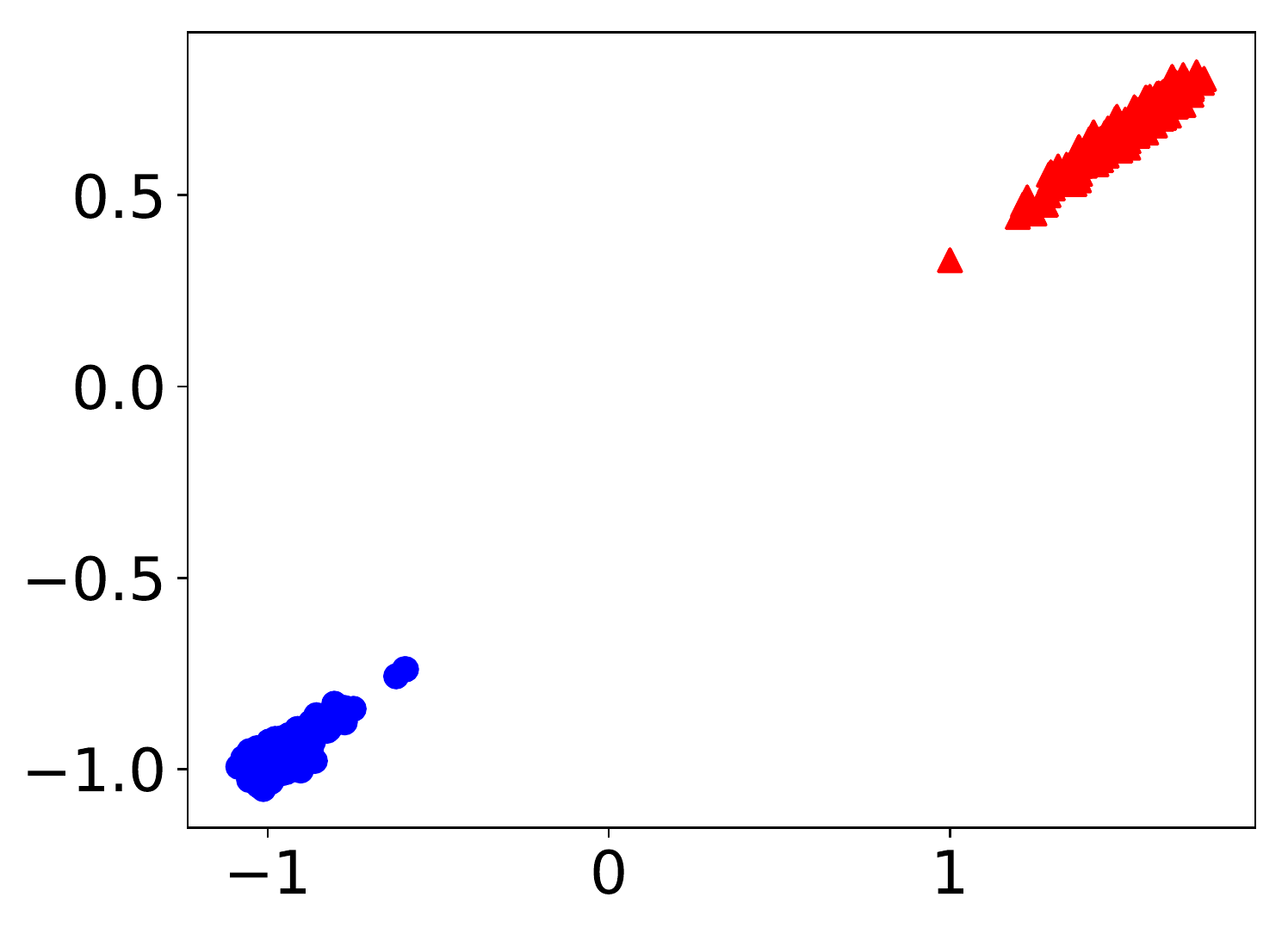}
  \label{fig:sfig2}\vspace{-5mm}
  \caption{\DNA, step 1}
\end{subfigure}
\begin{subfigure}{.23\textwidth}
  \centering
  \includegraphics[width=1\linewidth]{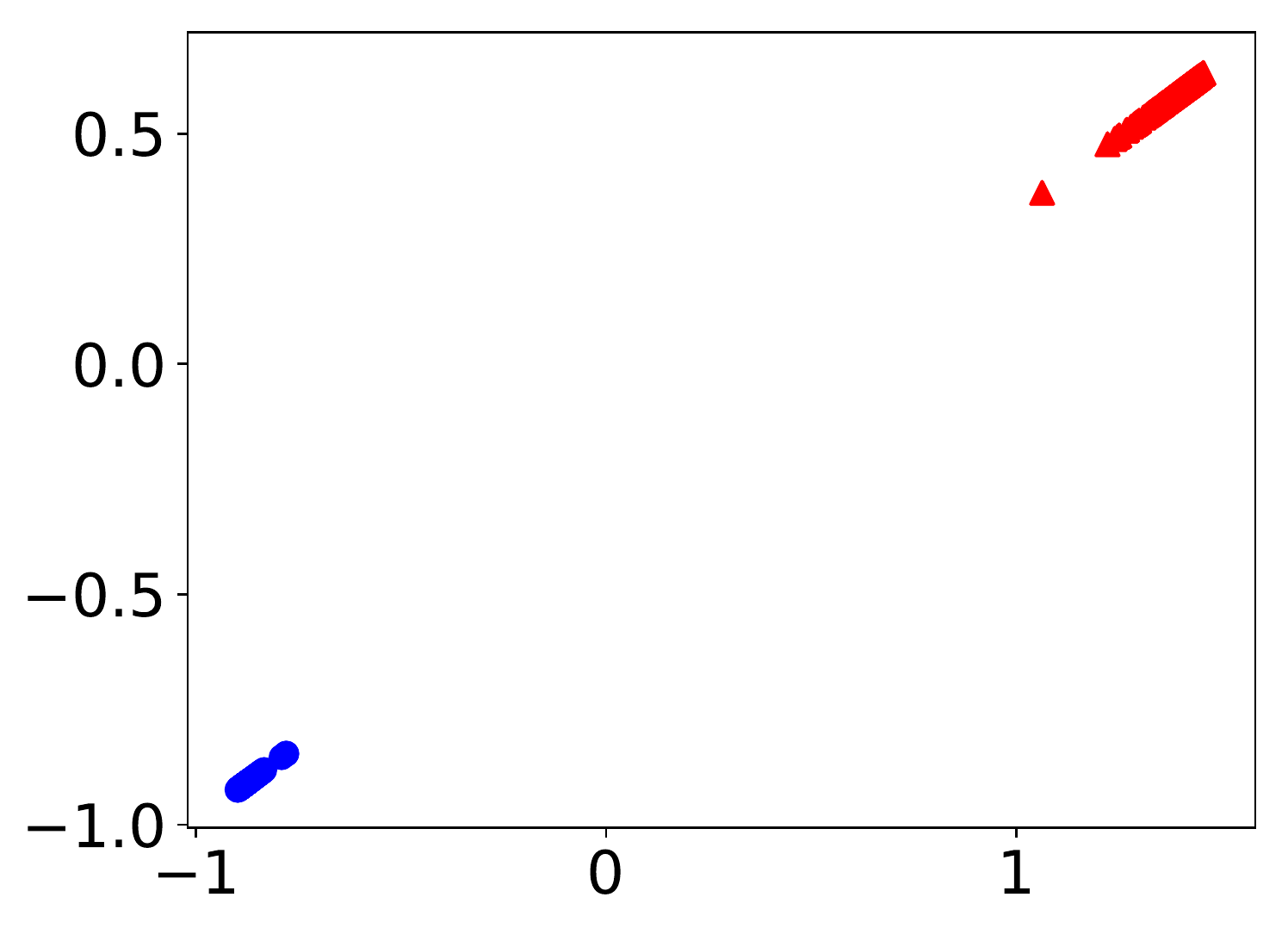}
  \label{fig:sfig2}\vspace{-5mm}
  \caption{\DNA, step 2}
\end{subfigure}
\begin{subfigure}{.23\textwidth}
  \centering
  \includegraphics[width=1\linewidth]{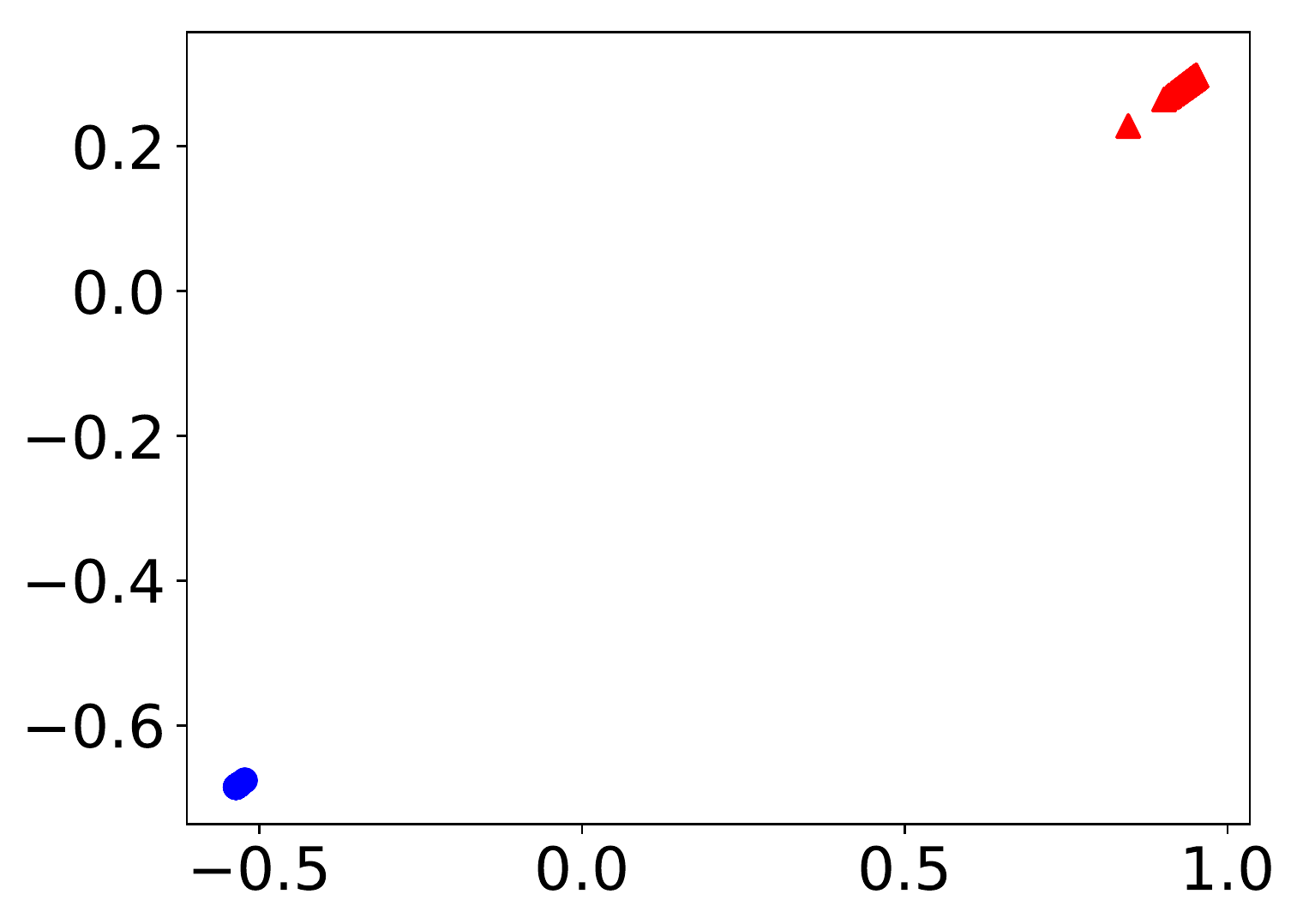}
  \label{fig:sfig1}\vspace{-5mm}
  \caption{\DNA, step 4}
\end{subfigure}%
\vspace{-2mm}
\caption{Mode-collapsing behavior on balanced mixture of Gaussian data: \UNA~and \DNA~behave similarly without mode collapsing after 4 steps.}
\label{fig:mode_collapse_2d_balance}
\end{figure}

\begin{figure}[hbt!]
\begin{subfigure}{.23\textwidth}
  \centering
  \includegraphics[width=1\linewidth]{figs/ub_data.pdf}
  \label{fig:sfig1}\vspace{-5mm}
  \caption{\UNA, step 0}
\end{subfigure}%
\begin{subfigure}{.23\textwidth}
  \centering
  \includegraphics[width=1\linewidth]{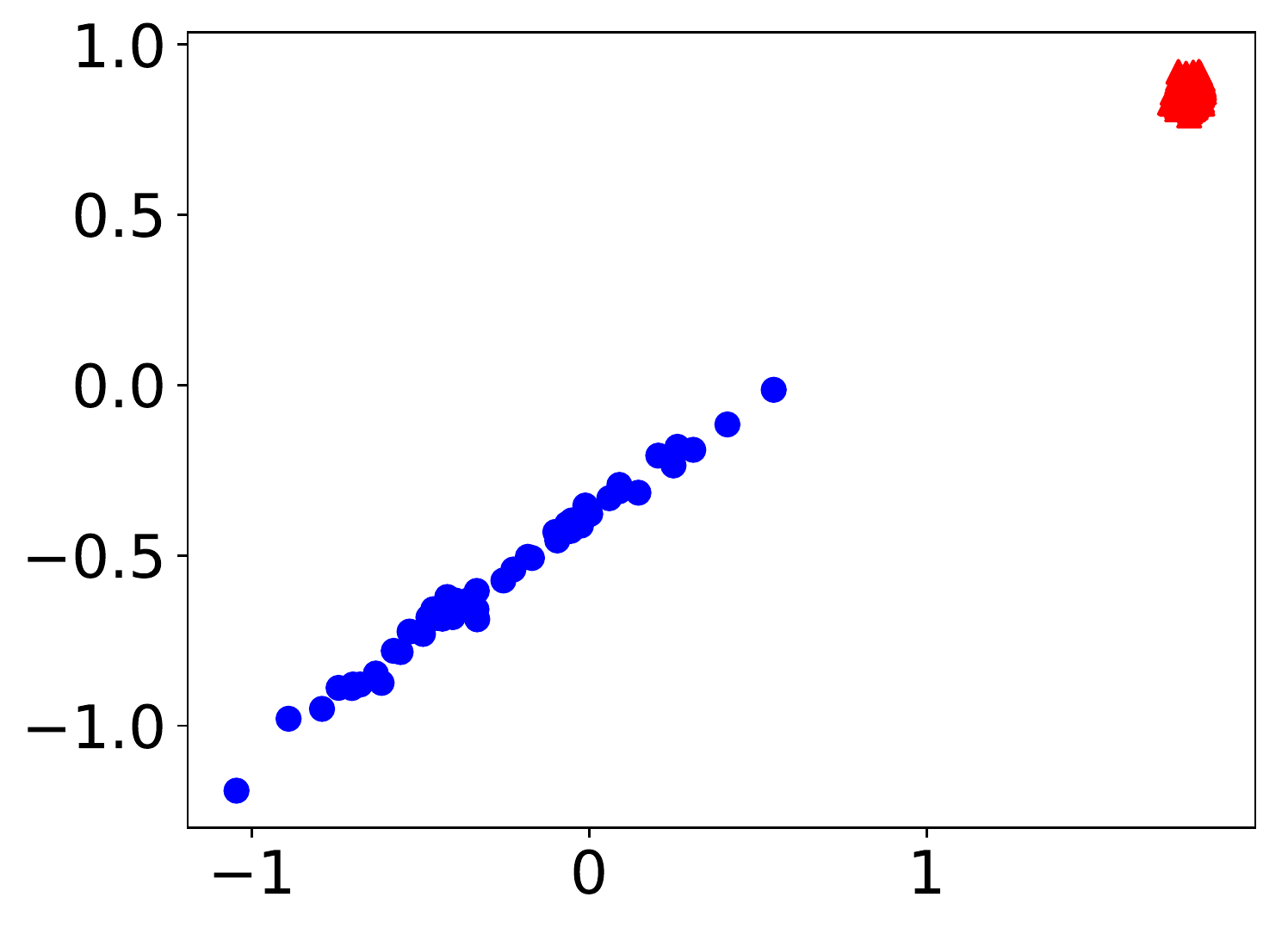}
  \label{fig:sfig2}\vspace{-5mm}
  \caption{\UNA, step 1}
\end{subfigure}
\begin{subfigure}{.23\textwidth}
  \centering
  \includegraphics[width=1\linewidth]{figs/ub_LN_step2.pdf}
  \label{fig:sfig2}\vspace{-5mm}
  \caption{\UNA, step 2}
\end{subfigure}
\begin{subfigure}{.23\textwidth}
  \centering
  \includegraphics[width=1\linewidth]{figs/ub_LN_step4.pdf}
  \label{fig:sfig1}\vspace{-5mm}
  \caption{\UNA, step 4}
\end{subfigure}%

\begin{subfigure}{.23\textwidth}
  \centering
  \includegraphics[width=1\linewidth]{figs/ub_data.pdf}
  \label{fig:sfig1}\vspace{-5mm}
  \caption{\DNA, step 0}
\end{subfigure}%
\begin{subfigure}{.23\textwidth}
  \centering
  \includegraphics[width=1\linewidth]{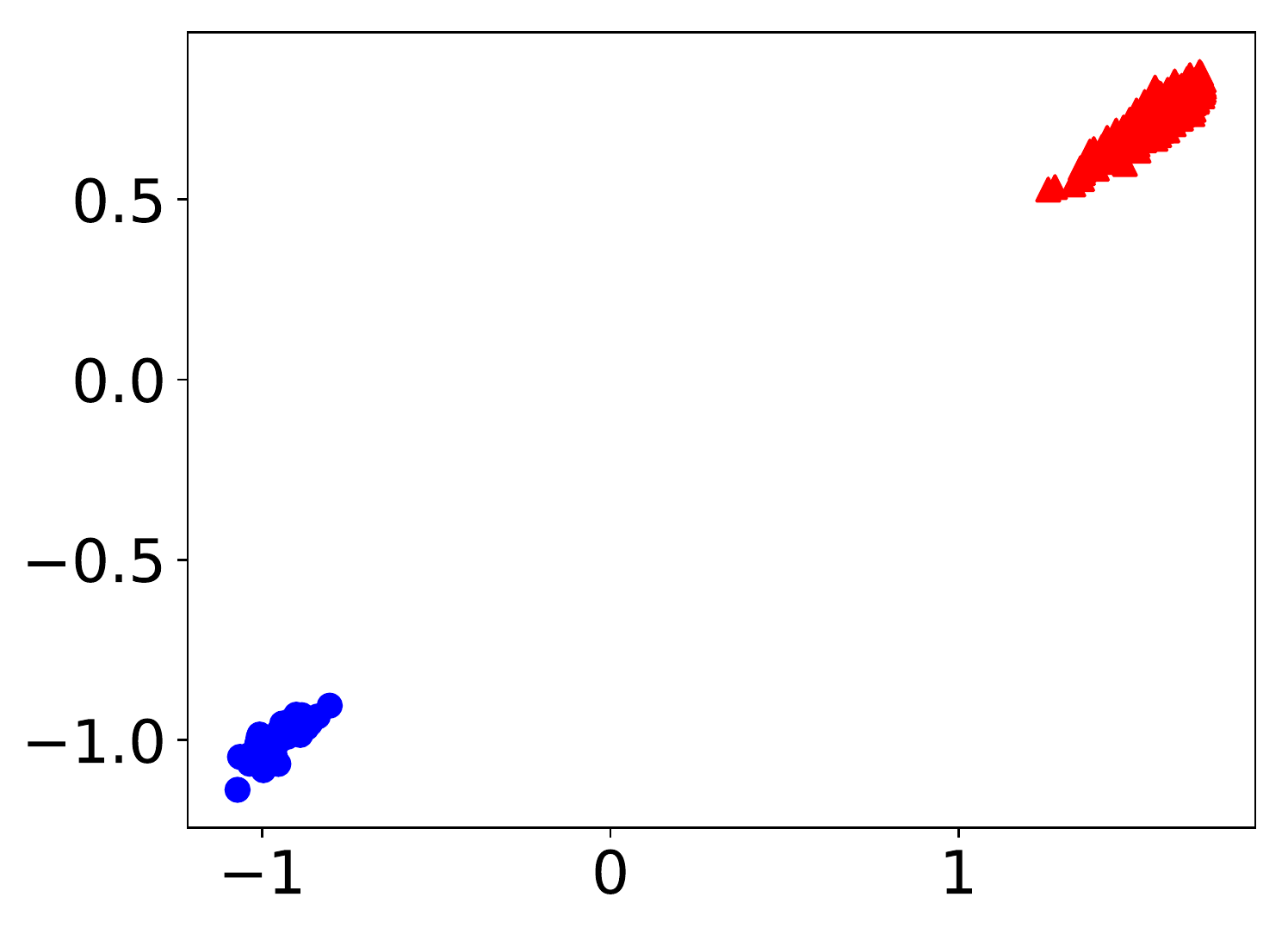}
  \label{fig:sfig2}\vspace{-5mm}
  \caption{\DNA, step 1}
\end{subfigure}
\begin{subfigure}{.23\textwidth}
  \centering
  \includegraphics[width=1\linewidth]{figs/ub_UN_step2.pdf}
  \label{fig:sfig2}\vspace{-5mm}
  \caption{\DNA, step 2}
\end{subfigure}
\begin{subfigure}{.23\textwidth}
  \centering
  \includegraphics[width=1\linewidth]{figs/ub_UN_step4.pdf}
  \label{fig:sfig1}\vspace{-5mm}
  \caption{\DNA, step 4}
\end{subfigure}%
\vspace{-2mm}
\caption{Mode-collapsing behavior on unbalanced mixture of Gaussian data: \UNA~collapses to one cluster after 4 steps, while \DNA~maintains 2 clusters.}
\label{fig:mode_collapse_2d_unbalance2}
\end{figure}

\end{document}